\documentclass[conference]{IEEEtran}
\IEEEoverridecommandlockouts
\usepackage{cite}
\usepackage{amsthm,amsmath,amssymb}
\usepackage{algorithmic}
\usepackage{algorithm}
\usepackage{graphicx} 
\usepackage{xcolor}
\usepackage{type1cm}
\newtheorem{proposition}{Proposition} 
\newtheorem{theorem}{Theorem}
\newtheorem{lemma}{Lemma}
\newtheorem{collorary}{Collorary}
\newtheorem{definition}{\textbf{Definition}} 
 
\def\BibTeX{{\rm B\kern-.05em{\sc i\kern-.025em b}\kern-.08em
    T\kern-.1667em\lower.7ex\hbox{E}\kern-.125emX}}
\begin{document} 

\title{Gathering of seven autonomous mobile robots \\on triangular grids}

\author{\IEEEauthorblockN{Masahiro Shibata}
\IEEEauthorblockA{\textit{Graduate School of Computer Science and Systems Engineering} \\ 
\textit{Kyushu Institute of Technology}\\
Fukuoka, Japan \\
shibata@cse.kyutech.ac.jp}
\and
\IEEEauthorblockN{Masaki Ohyabu}
\IEEEauthorblockA{\textit{Graduate School of Computer Science and Engineering} \\
\textit{Nagoya Institute of Technology}\\
Aichi, Japan \\
oyabu@moss.elcom.nitech.ac.jp}
\and
\IEEEauthorblockN{Yuichi Sudo}
\IEEEauthorblockA{\textit{Graduate School of Information Science and Technology} \\
\textit{Osaka University}\\
Osaka, Japan \\
y-sudou@ist.osaka-u.ac.jp}
\and
\IEEEauthorblockN{Junya Nakamura}
\IEEEauthorblockA{\textit{Information and Media Center} \\
\textit{Toyohashi University of Technology}\\
Aichi, Japan \\
junya@imc.tut.ac.jp}
\and
\IEEEauthorblockN{Yonghwan Kim}
\IEEEauthorblockA{\textit{Graduate School of Computer Science and Engineering} \\
\textit{Nagoya Institute of Technology}\\
Aichi, Japan  \\
kim@nitech.ac.jp}
\and
\IEEEauthorblockN{Yoshiaki Katayama}
\IEEEauthorblockA{\textit{Graduate School of Computer Science and Engineering} \\
	\textit{Nagoya Institute of Technology}\\
Aichi, Japan \\
katayama@nitech.ac.jp}

}

\maketitle

\begin{abstract}
In this paper, we consider the gathering problem of seven autonomous mobile  robots 
on triangular grids.
The gathering problem requires that, starting  from any connected initial configuration
where a subgraph induced by all robot nodes (nodes where a robot exists) constitutes one connected graph,
robots  reach a configuration such that 
the maximum distance between two robots is minimized.
For the case of seven robots, 
gathering is achieved when one robot has six adjacent robot nodes
(they form a shape like a hexagon).
In this paper, we aim to clarify the relationship between 
the capability of robots and the solvability of 
gathering 
on a triangular grid.
In particular, we focus on visibility range of robots. 
To discuss the solvability of the problem in terms of the visibility range, 
we consider strong assumptions except for visibility range. 
Concretely, we assume that robots are fully synchronous and 
they agree on the direction and orientation of the $x$-axis, and 
chirality in the triangular grid.
In this setting, we first consider the weakest assumption about visibility range, i.e.,  
robots with visibility range 1. 
In this case, we show that there exists no collision-free algorithm to solve the gathering problem.
Next, we extend the visibility range to 2.
In this case, we show that our algorithm can solve the problem from any connected initial configuration.
Thus, the proposed algorithm is optimal in terms of visibility range.

\end{abstract}


\section{Introduction}
\label{intro}
\subsection{Background}
Studies for (autonomous) mobile robot systems have emerged recently in the field of Distributed Computing. 
Robots aim to achieve some tasks with limited capabilities. 
Most studies assume that robots are uniform  
(they execute the same algorithm and cannot be distinguished by their appearance) and
oblivious (they cannot remember their past actions). 
In addition, it is assumed that robots cannot communicate with other robots explicitly. 
Instead, the communication is done implicitly;
each robot can observe the positions of the other robots.

\subsection{Related work}

Since Suzuki and Yamashita presented the pioneering work~\cite{YamashitaSensei}, 
many problems have been studied. 
For example, the gathering problem, 
which requires all robots to meet at a non-predetermined single point,
has been studied in various environments.
In the two-dimensional Euclidean space (a.k.a., the continuous model),
Suzuki and Yamashita \cite{YamashitaSensei} showed that 
when robots are not fully synchronous, 
the deterministic gathering of two robots is impossible without additional assumptions. 
This impossibility result was generalized to an even number of robots initially located evenly 
at two positions by Courtieu et al.~\cite{courtieuipl}. 
By contrast, Dieudonn\'e and Petit~\cite{continuous1} showed that,
by adding the assumption that robots can count the exact number of robots at each position, 
an odd number of robots can gather from any initial position. 

The gathering problem in the discrete space (a.k.a., the graph model) has also been studied.
In the discrete space, robots 
stay at fixed positions (the nodes of the graph), 
and move from one position to the next position 
through edges of the graph.
For (square) grid graphs, D'Angelo et al.~\cite{gridGathering}  and 
Castenow et al.~\cite{gridGathering2} proposed algorithms to solve 
the gathering problem. 
For ring graphs, 
Klasing et al.~\cite{descrete1,descrete2} 
showed the existence of unsolvable initial configurations and 
proposed algorithms to solve the problem from some specific initial configurations.
D'Angelo et al. \cite{descrete3} proposed an algorithm to solve the problem from
any solvable initial configuration.
Stefano and Navarra~\cite{ringGatheringMoves} analyzed the required total number of robot moves to solve the gathering problem in rings. 

\if()
Izumi et al.~\cite{descrete4} provided a deterministic gathering algorithm 
under the assumption that 
initial configurations are non-symmetric or non-periodic, and that the number of robots is less than half of the number of nodes. 
For odd number of robots or odd number of nodes,
Kamei et al.~\cite{descrete5,descrete6} proposed gathering algorithms that also work when started from symmetric configurations. 
\fi

	As a variant of mobile robots, 
gathering of \textit{fat robots} is considered \cite{fat1,fat2,fatGrid}.
Each fat robot dominates a space of a unit disc.
There are several definitions of the gathering problem for fat robots, e.g., 
robots achieve gathering when 
(i) they form a connected configuration 
(each robot touches at least one other robot and 
all robots form one connected formation) or  
(ii) they  reach a configuration such that 
the maximum distance between two robots is minimized.
For both the definitions, a collision is not allowed.
Thus, introducing sizes gives several definitions of the gathering problem, 
which is an interesting point.
Czyzowicz et al. \cite{fat1} considered gathering of (i) for three or four fat robots 
in the continuous model, and 
Chrysovalandis et al. \cite{fat2}
studied gathering of (i) for arbitrary number of robots.
Ito et al. \cite{fatGrid} considered gathering of (ii) on discrete square grids.	

Recently, one of computational models for programmable matter, 
\textit{amoebot} has been introduced \cite{amoeFirst}.
Each amoebot moves on a triangular grid and occupies one or two adjacent nodes.
Each amoebot has a finite memory, limited visibility range, and ability to communicate with 
a robot staying at an adjacent node.
Several problems using amoebots have been considered, such as leader election \cite{amoeLeader2}, 
gathering \cite{amoeGathering}, and 
shape formation (or pattern formation) \cite{YamauchiSensei,shape2}.
Recall that while amoebots have finite memory and communication capability,
(standard) autonomous mobile robots have no memory or communication capability.
Hence, the mobile robot model is weaker than 
the amoebot model, and 
it is interesting to clarify solvability of problems between 
the mobile robot model and the amoebot model.
 
Meanwhile, when considering a discrete space,
a space filled by regular polygons is sometimes preferable because 
its simple structure helps to design an algorithm and 
to discuss the solvability of a problem among various robot models.
In addition, (i) only triangular,  square, and  hexagonal grids are 
discrete spaces filled by  regular polygons, 
(ii) gathering on a square gird has already been studied \cite{fatGrid}, and 
(iii) recently the amoebot model has been extensively studied on a triangular grid.
Hence, in this paper we consider  gathering of mobile robots on a triangular grid. 

\subsection{Our contribution}

In this paper, 
we consider the gathering problem of seven mobile robots
on triangular grids.
We say in this paper that 
gathering is achieved when robots reach a configuration such that 
the maximum distance between two robots is minimized.
For the case of seven robots, letting a \textit{robot node} be a node where a robot exists,
gathering is achieved when one robot has six adjacent robot nodes
like Fig.\,\ref{fig:example}.
This implies that 
 robots form a (filled) hexagon.
In this paper, we aim to clarify the relationship between 
the capability of robots and the solvability of the gathering problem on a triangular grid.
In particular, we focus on visibility range of robots. 
To discuss the solvability of the problem in terms of the visibility range, 
we consider strong assumptions except for visibility range. 
Concretely, we assume that robots are fully synchronous, and they 
agree on the direction and orientation of the $x$-axis, and 
chirality in the triangular grid.
In this setting, we first consider the weakest assumption about visibility range, i.e.,  
robots with visibility range 1. 
In this case, we show that there exists no collision-free algorithm to solve the gathering problem.
Next, we extend the visibility range to 2.
In this case, we show that our algorithm can solve the problem from any connected initial configuration.
Thus, the proposed algorithm is optimal in terms of visibility range.

\begin{figure}[t!]
		\centering 
		\includegraphics[scale=0.51]{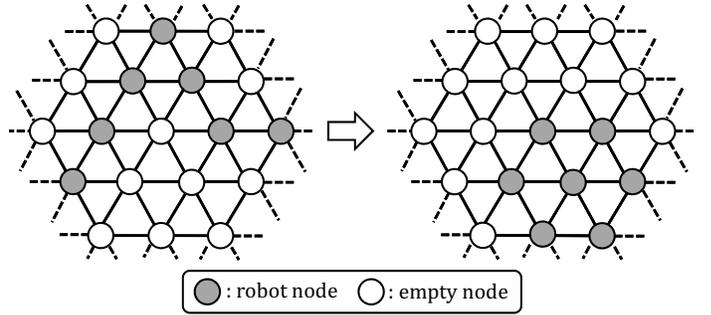}
		\caption{An example of the gathering problem.}
		\label{fig:example}
\end{figure}

\section{Preliminaries}
\subsection{System model}
An (infinite) triangular grid is an undirected graph $G=(V,E)$, where 
$V$ is the set of nodes and $E$ is the set of edges. 
The grid has one special node called \textit{origin}, and we denote it by $v_o$.
Each node $v_j\in V$ has six \textit{adjacent nodes}:
east ($v_E^j$ or E), southeast ($v_{\textit{SE}}^j$ or SE), 
southwest ($v_{\textit{SW}}^j$ or SW), west ($v_W^j$ or W), 
northwest ($v_{\textit{NW}}^j$ or NW), and northeast ($v_{\textit{NE}}^j$ or NE).
The axis including $v_o$ and $v_E^o$ (resp., $v_o$ and $v_\textit{NE}^o$)
is called the \textit{$x$-axis}
(resp., \textit{$y$-axis})\footnote{Although the origin, the $x$-axis, and the $y$-axis are terms of the coordinate system, 
	we use these terms for explanation.
	In the following, we use several terms of the coordinate system.}.
An example is given in Fig.\,\ref{fig:grid}.
In addition, 
a sequence of $k+1$ distinct nodes $(v_0,v_1,\dots,v_k)$ is called a \textit{path} with length $k$ if $\{v_i,v_{i+1}\} \in E$ for all $i \in [0,k-1]$.
The \textit{distance} between two nodes is defined as 
the length of the shortest path between  them.

In this paper, we consider seven mobile robots  and 
denote the robot set by $R =\{r_0, r_1,\ldots , r_6\}$.
Robots considered here have the following characteristics. 
Robots are \textit{uniform}, that is, 
they execute the same algorithm and cannot be distinguished by their appearance. 
Robots are \textit{oblivious}, that is, 
they have no persistent memory and cannot remember their past actions. 
Robots cannot communicate with other robots directly.
However, robots have limited visibility range and 
they can observe the positions of other robots within the range. 
This means that robots can communicate implicitly by their positions. 
We consider two problem settings about robots: 
\textit{robots with visibility range 1} and \textit{robots with visibility range 2}.
Robots with visibility range 1 can observe nodes within distance 1, that is, 
they can only observe their six adjacent nodes.
On the other hand, 
robots with visibility range 2 can observe nodes within distance 2 (eighteen nodes in total).
We assume that they are transparent, that is,
even if a robot $r_i$ and several robots exist on the same axis,
$r_i$ can observe all the robots on the axis within its visibility range.
Robots do not know the position of the origin, but 
they agree on the direction and orientation of the $x$-axis, and 
chirality (the orientation of axes, e.g., clockwise or counter-clockwise)
in the triangular gird.

\begin{figure}[t!]
			\centering
			\includegraphics[scale=0.475]{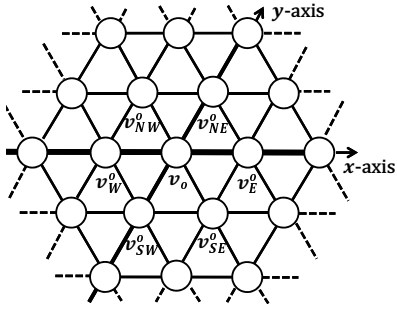}
			\caption{An example of a triangular grid.}
			\label{fig:grid}
\end{figure}

Each robot executes the algorithm by repeating \textit{Look-Compute-Move cycles}. 
At the beginning of each cycle, 
the robot observes 
positions of the other robots within its visibility range
(Look phase). 
According to the observation, 
the robot computes whether it moves to its adjacent node or 
stays at the current node (Compute phase). 
If the robot decides to move, it moves to the node by the end of the cycle (Move phase). 
Robots are fully synchronous (FSYNC), that is, 
all robots start every cycle simultaneously and 
execute each phase synchronously.
We assume that a \textit{collision} is not allowed during execution of the algorithm.
Here, a collision represents a situation such that 
two robots traverse the same edge from different directions
or several robots exist at the same node.
Concretely, the following three behaviors are not allowed:
(a) some robot $r_i$ (resp., $r_j$) staying at node $v_p$ (resp., $v_q$) moves to 
$v_q$ (resp., $v_p$),
(b) some robot $r_i$ staying at node $v_p$ remains at $v_p$ and 
robot $r_j$ staying at 
node $v_q$ moves to $v_p$, and
(c) several robots move to the same empty node.

A \textit{configuration} of the system is defined as the set of  locations of each robot.
Here, the location of a robot $r_i$ is   
defined as the position that 
(1) $r_i$ is currently staying at and 
(2) is represented as an intersection  of an axis parallel to the $x$-axis and 
an axis parallel to the $y$-axis.
Each axis $\textit{ax}$ is represented by 
(i) whether it is parallel to the $x$-axis or the $y$-axis and 
which direction it is far from the axis, and 
(ii) the number of axes between 
the axis including $v_o$ (i.e., the $x$-axis or the $y$-axis) and \textit{ax}.
However, robots do not know the position of $v_o$ and 
they cannot use information of (global) locations.
A node is called a \textit{robot node} if the node is occupied by a robot.
Otherwise, the node is called an \textit{empty node}.
We assume that the initial configuration is connected, that is, 
the subgraph of $G$ induced by the seven robot nodes is connected.
This assumption of connectivity is necessary because, 
if a configuration becomes unconnected and a robot $r$ has no adjacent robot node,
$r$ cannot know the direction to reconstruct a connected configuration
due to 
obliviousness, which implies that robots cannot achieve gathering.

When a robot executes a Look phase, it gets a \textit{view} of the system. 
A view of a robot is defined 
as the set of robot nodes within its visibility range.
For example, in Fig.\,\ref{fig:configuration},
a robot at node $v_j$ recognizes that 
nodes $v_E^j, v_\textit{SW}^j$, and $v_\textit{NE}^j$ are robot nodes
when its visibility range is 1 
and recognizes that nodes $v_k$ and $v_\ell$ are also robot nodes when its visibility range is 2.

\begin{figure}[t!]
			\centering
			\includegraphics[scale=0.475]{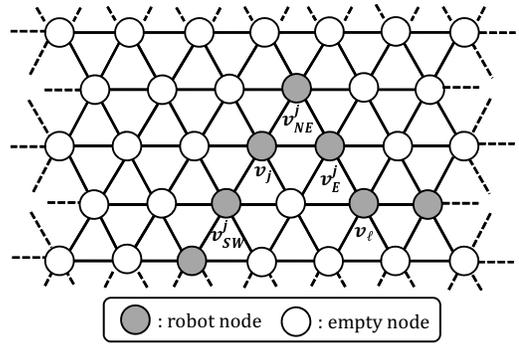}
			\caption{An example of a configuration.}
			\label{fig:configuration}
\end{figure}

\subsection{Gathering problem}
The gathering problem of mobile robots requires that  starting 
from any connected initial configuration, 
the robots terminate in  a configuration such that 
the maximum distance between two robot nodes is minimized.
In the case of seven robots, gathering is achieved when
one robot has six adjacent robot nodes
(Fig.\,\ref{fig:example}). 
Concretely, we define the problem as follows.

\begin{definition}
	A collision-free algorithm $\cal{A}$ solves the gathering problem of seven autonomous mobile robots  on a triangular grid if and only if
	the system reaches a configuration such that 
	one robot has six adjacent robot nodes and no robot moves thereafter, without a collision throughout the  execution of $\cal{A}$, 
\end{definition}

\section{Robots with visibility range 1}
In this section, for robot with visibility range 1,
we show that there exists no collision-free algorithm to solve 
the problem.

\begin{theorem}
	\label{theo:insolvability}
	For robots with visibility range 1, 
	there exists no collision-free algorithm to solve the gathering problem even in the fully synchronous (FSYNC) model.
\end{theorem}

\begin{proof}
We show the proof by contradiction, that is, we assume that 
there exists a collision-free algorithm $\cal{A}$ to solve the gathering problem 
from any connected initial configuration.
In the proof, we consider several configurations and robot behaviors, and show that 
if some robot moves to some direction by Algorithm $\cal{A}$,
several robots cannot move anywhere
(i.e., they have to stay at the current nodes) 
since a collision occurs or the configuration becomes unconnected.
Eventually, we show that 
there is a configuration such that 
all robots need to stay at the current nodes and they cannot achieve gathering,
which is a contradiction.

First, we consider the configuration of Fig.\,\ref{fig:slash} (a).
In the figure, robot $r_i$ (resp., $r_j$) has one adjacent robot node SE (resp., NW) and 
the other robots have two adjacent robot nodes SE and NW, respectively. 
In such a configuration,  we first show that intermediate robots cannot leave the current nodes.

\begin{figure}[t!]
	\centering 
		\includegraphics[scale=0.45]{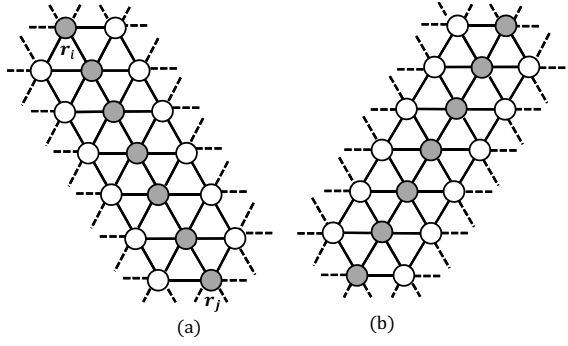}
		\caption{Configurations we consider in the proof.}
		\label{fig:slash}
\end{figure}

\begin{lemma}
	\label{lem:intermediate}
	A robot with two adjacent robot nodes W and E, SW and NE, or NW and SE 
	must stay at the current node. 
\end{lemma}

\begin{proof}
	We consider configurations of Fig.\,\ref{fig:degree2}.
	In each configuration, 
	robots $r_i$ and $r_j$ have two adjacent robot nodes W and E.
	On the other hand, robots $r_p$ and $r_q$ have three adjacent robot nodes and 
	they must stay at the current nodes because 
	they cannot detect whether 
	the current configuration is a gathering-achieved configuration or not.
	In addition, if $r_i$ moves to W, NW, or SW, $r_j$ also moves to the same direction because 
	they have the same view.
	Then, either in Fig.\,\ref{fig:degree2} (a) or (b),
	wherever  $r_k$ moves to, 
	a collision occurs  or the configuration becomes unconnected.
	By a similar discussion, 
	when $r_i$ and $r_j$ move to E, NE, or SE, 
	it causes a collision or an unconnected configuration 
	either in Fig.\,\ref{fig:degree2} (c) or (d).
	Thus, a robot with two adjacent robot nodes E and W cannot leave the current node. 
	By the similar discussion,  we can show that 
	a robot with two adjacent robot nodes SW and NE, or NW and SE must stay at the current node.
	Thus, the lemma follows. 
\end{proof}

By this lemma, we can have the following two colloraries.

\begin{collorary}
	\label{collo:oneNeighbor}
	A robot with one adjacent robot node E, SE, SW, W, NW, or NE 
	can move only to NE or SE, E or SW, SE or W, SW or NW, W or NE, or NW or E
	if it moves, respectively.
\end{collorary}

\begin{collorary}
	\label{collo:twoNeighbor}
	A robot with two adjacent robot nodes E and SW, SE and W, SW and NW, W and NE, 
	NW and E, or NE and SE can move only to node SE, SW, W, NW, NE, or E if it moves, 
	respectively.
\end{collorary}

By Lemma \ref{lem:intermediate}, 
 intermediate robots  in Fig.\,\ref{fig:slash} (a) cannot leave the current nodes, and hence 
$r_i$ or $r_j$ 
has to leave the current node.
Without loss of generality, we assume that in $\cal{A}$
robot $r_i$ with one adjacent robot node SE moves to SW. 
Notice that $r_i$ can move only to SW or E by Collorary \ref{collo:oneNeighbor}.
In the following, we consider several robot behaviors and eventually show that 
a robot with one adjacent robot node NE or SW must stay at the current node.
Then, in a configuration of Fig.\,\ref{fig:slash} (b),
all robots must stay at the current nodes and 
they cannot solve the gathering problem, which is a contradiction.

\begin{figure}[t!] 
	\centering
		\includegraphics[scale=0.525]{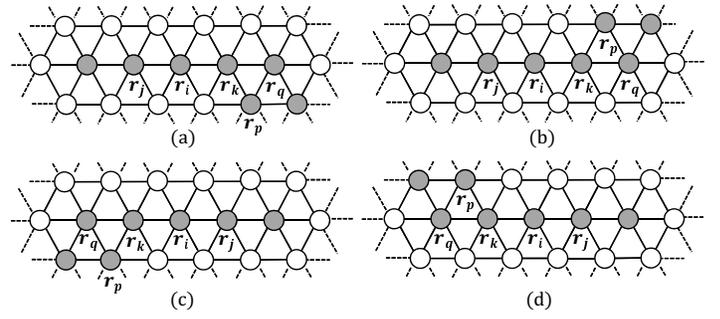}
		\caption{An example of a configuration that 
			a robot with two adjacent robot nodes W and E must stay at the current node.}
		\label{fig:degree2}
\end{figure}

When a robot with one adjacent robot node SE moves to SW,
several robot behaviors are not allowed since a collision occurs,
as shown in Fig.\,\ref{fig:assumptionCounter}
(for simplicity, we omit robot nodes unrelated to prohibited robot behaviors).
Concretely, we have the following proposition. 

\begin{proposition}
	\label{pro:SEtoSW}
	When a robot with one adjacent robot node SE moves to SW, 
	the following four robot behaviors are not allowed:
	(a) a robot with one adjacent robot node NE moves to NW,
	(b) a robot with two adjacent robot nodes NW and SW moves to W,
	(c) a robot with one adjacent robot node E moves to NE, and 
	(d) a robot with two adjacent robot nodes NW and E moves to NE.
\end{proposition}

In the following, we consider the following five cases: 
(1) a robot with one adjacent robot node NW moves to W,
(2) a robot with one adjacent robot node SW moves to SE, 
(3) a robot with one adjacent robot node NE moves to E,
(4) a robot with one adjacent robot node NW moves to NE, and 
(5) a robot with one adjacent robot node SW moves to W.
In each case, we show that the assumed robot behavior is not allowed.
Thus, by Proposition \ref{pro:SEtoSW}-(a) and cases (2), (3), and (5), 
robots cannot achieve gathering from the configuration of Fig.\,\ref{fig:slash} (b),
which is a contradiction
(results of cases (1) and (4) are used for cases (2), (3), and (5)).

\begin{figure}[t!]
	\centering 
	\includegraphics[scale=0.545]{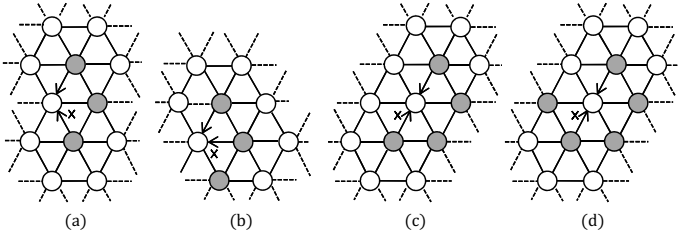}
	\caption{Prohibited behaviors when 
		a robot with one adjacent robot node SE moves to SW.}
	\label{fig:assumptionCounter}
\end{figure}

\textbf{Case 1: a robot with one adjacent robot node NW moves to W.} 
In this case, as shown in Fig.\,\ref{fig:counter1},
the following three robot behaviors are not allowed:
(a) a robot with two adjacent robot nodes W and SE moves to SW, 
(b) a robot with one adjacent robot node E moves to SE, and 
(c) a robot with one adjacent robot node NE moves to E.
Then, let us consider the configuration of Fig.\,\ref{fig:case1}.
In the configuration, by Proposition \ref{pro:SEtoSW} and the above discussion,
no robot can leave the current node and robots cannot achieve gathering, 
which is a contradiction.
Thus, we have the following lemma.

\begin{lemma}
	\label{lem:noNWtoW}
	A robot with one adjacent robot node NW cannot move to node W.  
\end{lemma}

\if()
\begin{figure}[t!]
	\centering
	\includegraphics[scale=0.54]{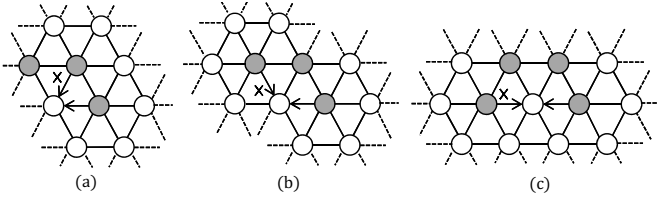}
	\caption{Prohibited behaviors when 
		a robot with one adjacent robot node NW moves to node W.}
	\label{fig:counter1}
\end{figure}
\fi

\begin{figure}[t!]
	\centering
	\includegraphics[scale=0.54]{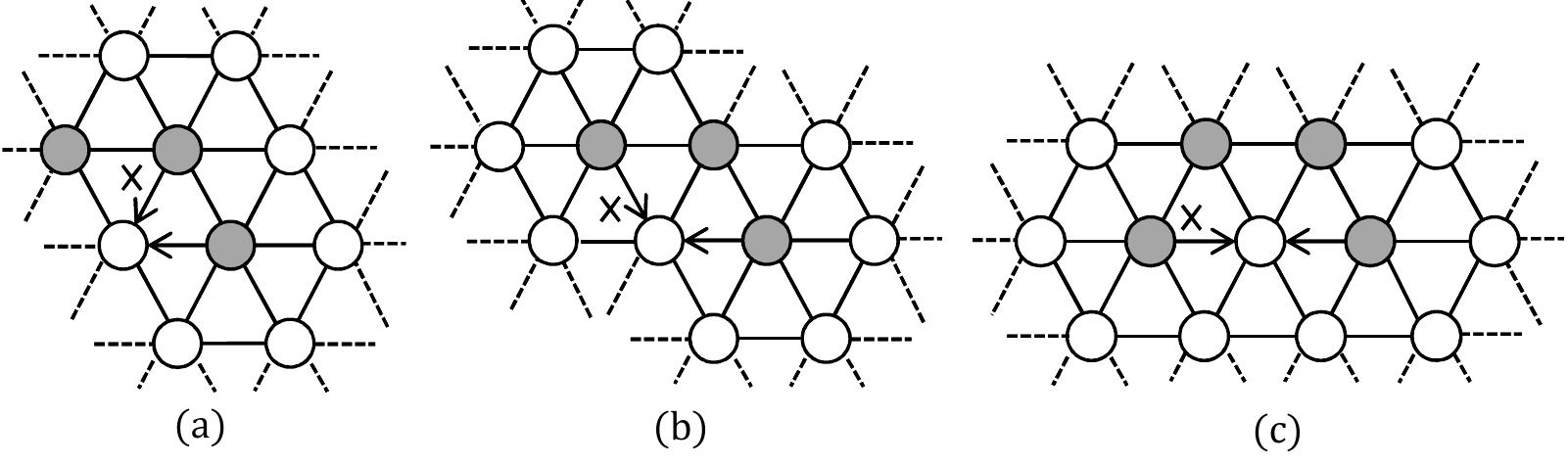}
	\caption{Prohibited behaviors when 
		a robot with one adjacent robot node NW moves to node W.}
	\label{fig:counter1}	
\end{figure}

\textbf{Case 2: a robot with one adjacent robot node SW moves to SE.}
In this case, as shown in Fig.\,\ref{fig:case2Counter}, 
the following four robot behaviors are not allowed: 
(a) a robot with one adjacent  robot node NW moves to NE,
(b) a robot with two adjacent robot nodes NE and SE moves to E,  
(c) a robot with one adjacent robot node W moves to NW, and 
(d) a robot with two adjacent robot nodes NW and E (resp., W and NE) moves to NE (resp., NW).
Then, in a configuration of Fig.\,\ref{fig:case2Mid},
only robot $r_p$ with two adjacent robot nodes SW and E can move to SE or 
robot $r_q$ with two adjacent robot nodes W and SE can move to SW
by the above discussion and Lemmas \ref{lem:intermediate} and \ref{lem:noNWtoW}
and Collorary \ref{collo:twoNeighbor}.
We consider each of the behaviors and show for both the cases that 
robots cannot achieve gathering from some configuration.

\begin{figure}[t!]
	\centering
	\includegraphics[scale=0.405]{{eps/case1}}
	\caption{An unsolvable configuration when  a robot with 
		one adjacent robot node NW moves to  W
		((i): by Fig.\,\ref{fig:assumptionCounter} (c),
		(ii): by Fig.\,\ref{fig:counter1} (b),
		(iii): by Fig.\,\ref{fig:counter1} (a),
		(iv): by Fig.\,\ref{fig:assumptionCounter} (b),
		(v): by Fig.\,\ref{fig:assumptionCounter} (a), 
		(vi): by Fig.\,\ref{fig:counter1} (c)).}
	\label{fig:case1}	
\end{figure}

\if()
\begin{figure}[t!]
	\begin{center}
		\includegraphics[scale=0.55]{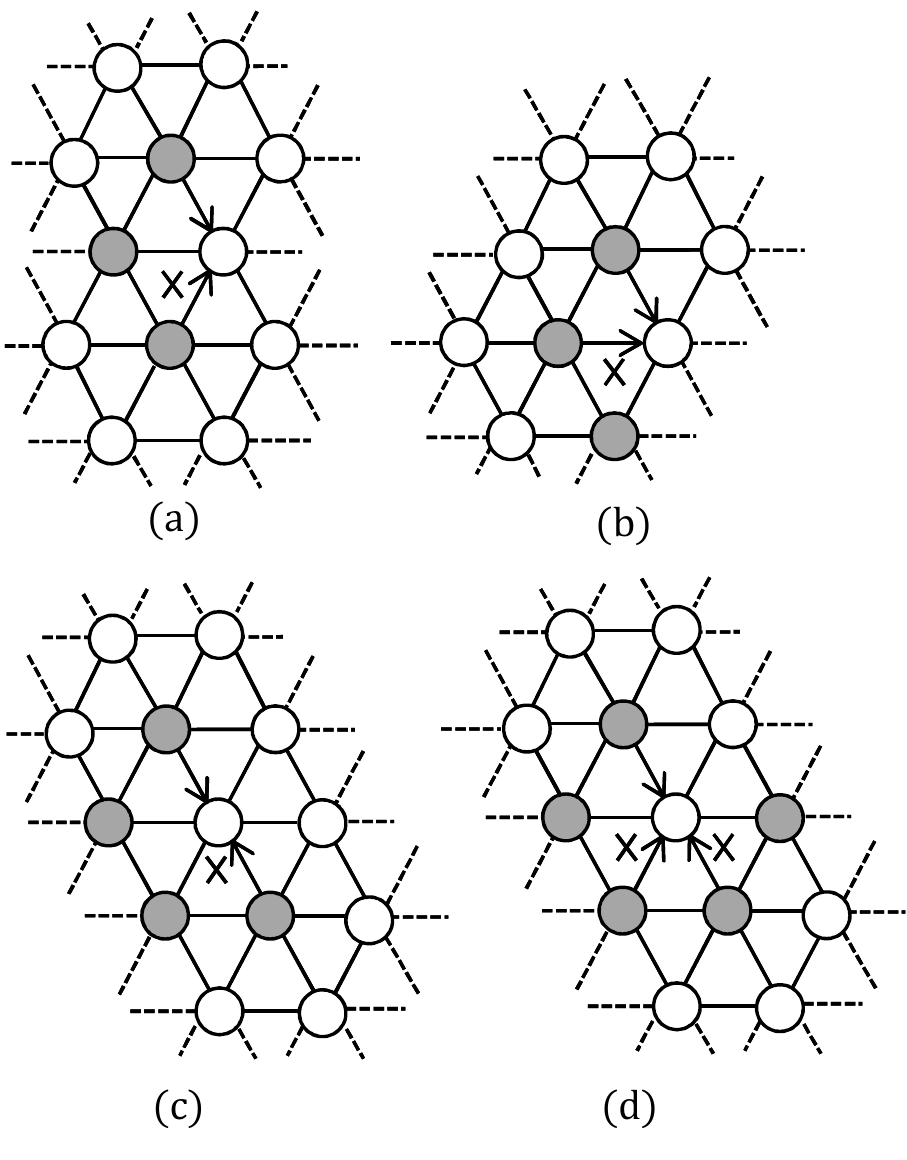}
		\caption{Prohibited behaviors under the assumption that
			a robot with one adjacent robot node SW moves to SE.}
		\label{fig:case2Counter}
	\end{center}
\end{figure}
\fi

\textit{Case 2-1: robot $r_p$ moves to  SE.}
In this case, 
clearly a robot with one adjacent robot node NE cannot move to E and
a robot with one adjacent robot node W cannot move to SW since a collision occurs.
In addition, when considering a configuration of Fig.\ref{fig:case2-1Mid} (a),
only robot $r_i$ with one adjacent robot node E can leave the current node and 
it needs to move to SE by the previous discussions.
Now, we consider the configuration of Fig.\,\ref{fig:case2-1} (a).
In the figure, robot $r_1, r_3$, and $r_5$ move to SE and 
the other robots must stay at the current nodes.
Then, the system reaches the configuration of Fig.\,\ref{fig:case2-1} (b).
In the configuration, robots $r_0, r_2, r_4$, and $r_6$ move to SE and 
the other robots must stay at the current nodes.
Then, the system reaches the configuration of Fig.\,\ref{fig:case2-1} (a).
Thus, robots repeat configurations of Fig.\,\ref{fig:case2-1} (a) and (b) forever and 
they cannot achieve gathering, which is a contradiction.

\begin{figure}[t!]
	
	\begin{minipage}{0.615\hsize}
		\centering 
		\includegraphics[scale=0.57]{eps/case2Counter}
		\caption{Prohibited behaviors when 
					a robot with one adjacent robot node SW moves to SE.}
		\label{fig:case2Counter}
	\end{minipage}
	\begin{minipage}{0.285\hsize}
		\centering
		\includegraphics[scale=0.45]{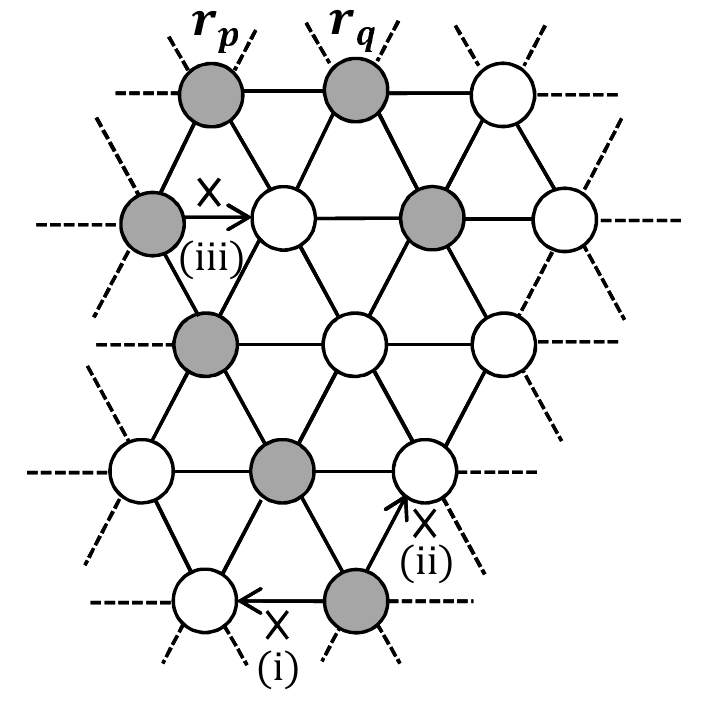}
		\caption{A configuration such that 
			only $r_p$ or $r_q$ can leave the current node 
			((i): by Lemma \ref{lem:noNWtoW}, 
			(ii): by Fig.\,\ref{fig:case2Counter} (a),  
			(iii): by Fig.\,\ref{fig:case2Counter} (b)).}
		\label{fig:case2Mid}
	\end{minipage}
	\begin{minipage}{0.01\hsize}
		\hspace{2mm}
	\end{minipage}
\end{figure}

\if()

\begin{figure*}[t!]

	\begin{minipage}{0.32\hsize}
		\begin{center}
			\includegraphics[scale=0.45]{eps/case2Mid}
			\caption{A configuration such that 
				only $r_p$ or $r_q$ can leave the current node 
				((i): by Lemma \ref{lem:noNWtoW}, 
				(ii): by Fig.\,\ref{fig:case2Counter} (a),  
				(iii): by Fig.\,\ref{fig:case2Counter} (b)).}
			\label{fig:case2Mid}
		\end{center}
	\end{minipage}
	\begin{minipage}{0.64\hsize}
	\begin{center}
			\includegraphics[scale=0.69]{{eps/case2-1Mid}}
	\caption{(a): An example such that only a robot $r_i$ with one adjacent robot node E can 
		leave the current node in Case 2-1
		((i): by Fig.\, \ref{fig:assumptionCounter} (c),
		(ii): by Fig.\,\ref{fig:case2Counter} (c),  
		(iii): prohibited robot behavior when 
		a robot with two adjacent robot nodes SW and E moves to SE), 
		(b): An example such that only a robot $r_i$ with one adjacent robot node W can 
		leave the current node in Case 2-2
		((iv): by Fig.\, \ref{fig:assumptionCounter} (c),
		(v): prohibited robot behavior when 
		a robot with two adjacent robot nodes W and SE moves to SW,  
		(vi): by Fig.\,\ref{fig:case2Counter} (c)).}
	\label{fig:case2-1Mid}
	\end{center}
\end{minipage}
\begin{minipage}{0.01\hsize}
	\hspace{2mm}
\end{minipage}
\end{figure*}

\fi

\begin{figure*}[t!]
	\centering
		\includegraphics[scale=0.69]{{eps/case2-1Mid}}
		\caption{(a): An example such that only a robot $r_i$ with one adjacent robot node E can 
			leave the current node in Case 2-1
			((i): by Fig.\, \ref{fig:assumptionCounter} (c),
			(ii): by Fig.\,\ref{fig:case2Counter} (c),  
			(iii): prohibited behavior when 
			a robot with two adjacent robot nodes SW and E moves to SE), 
			(b): An example such that only a robot $r_i$ with one adjacent robot node W can 
			leave the current node in Case 2-2
			((iv): by Fig.\, \ref{fig:assumptionCounter} (c),
			(v): prohibited  behavior  when 
			a robot with two adjacent robot nodes W and SE moves to SW,  
			(vi): by Fig.\,\ref{fig:case2Counter} (c)).}
		\label{fig:case2-1Mid}
\end{figure*}

\begin{figure*}[t!]
	\centering
		\includegraphics[scale=0.67]{{eps/case2-1}}
		\caption{Configurations that robots repeat alternately
			((i): by Fig.\,\ref{fig:case2-1Mid} (a),
			(ii): by Fig.\,\ref{fig:case2Counter} (d), 
			(iii): assumption of Case 2-1,
			(iv): assumption of Case 2,
			(v): by Fig.\,\ref{fig:assumptionCounter} (a),
			(vi):  prohibited  behavior when 
			a robot with two adjacent robot nodes SW and E moves to SE, 
			(vii): by Fig.\,\ref{fig:case2Counter} (c),  
			(viii): prohibited behavior when 
			a robot with two adjacent robot nodes SW and E moves to SE).
		}
		\label{fig:case2-1}
\end{figure*}

\begin{figure*}[t!]
	\centering
		\includegraphics[scale=0.67]{{eps/case2-2}}
		\caption{Configurations that robots repeat alternately
			((i): assumption of Algorithm $\cal{A}$,
			(ii): by Fig.\,\ref{fig:case2Counter} (d), 
			(iii): assumption of Case 2-2,
			(iv): by Fig.\,\ref{fig:case2-1Mid} (b),
			(v): by Fig.\,\ref{fig:assumptionCounter} (c), 
			(vi):  prohibited behavior when 
			a robot with two adjacent robot nodes W and SE moves to SW, 
			(vii): by Lemma \ref{lem:noNWtoW},  
			(viii): by Fig.\,\ref{fig:case2Counter} (a)).
		}
		\label{fig:case2-2}
\end{figure*}

\textit{Case 2-2: robot $r_q$ moves to  SW.}
In this case, 
clearly a robot with one adjacent robot node E cannot move to SE since a collision occurs.
In addition, when considering a configuration of Fig.\,\ref{fig:case2-1Mid} (b),
only robot $r_i$ with one adjacent robot node W can leave the current node and 
it needs to move to SW by the previous discussions.
Now, we consider the configuration of Fig.\,\ref{fig:case2-2} (a).
In the figure, robot $r_0, r_2, r_4$, and $r_6$ move to SW and 
the other robots must stay at the current nodes.
Then, the system reaches the configuration of Fig.\,\ref{fig:case2-2} (b).
In the configuration, robots $r_1, r_3$, and $r_5$ move to SW and 
the other robots must stay at the current nodes.
Then, the system reaches the configuration of Fig.\,\ref{fig:case2-2} (a).
Thus, robots repeat configurations of Fig.\,\ref{fig:case2-2} (a) and (b) forever and 
they cannot achieve gathering, which is a contradiction.
Thus, we have the following lemma.

\begin{lemma}
	\label{lem:noSWtoSE}
	A robot with one adjacent robot node \textit{SW} cannot move to node \textit{SE}. 
\end{lemma}

\textbf{Case 3: a robot with one adjacent robot node NE moves to E.} 
In this case, the following two robot behaviors are not allowed (Fig.\,\ref{fig:case3Counter}):
a robot with two adjacent robot nodes SW and E (resp., W and SE) 
moves to SE (resp., SW).
Then, in a configuration of Fig.\,\ref{fig:case3Mid},
it is necessary that at least 
a robot $r_p$ with two adjacent robot nodes NE and SE moves to E or 
a robot $r_q$ with one adjacent robot node with NW moves to NE
by the previous discussions.
We consider each of the behaviors and show for both the cases that 
robots cannot achieve gathering from some configuration.

\begin{figure}[t!]
	\begin{minipage}{0.47\hsize}
		\begin{center}
			\includegraphics[scale=0.475]{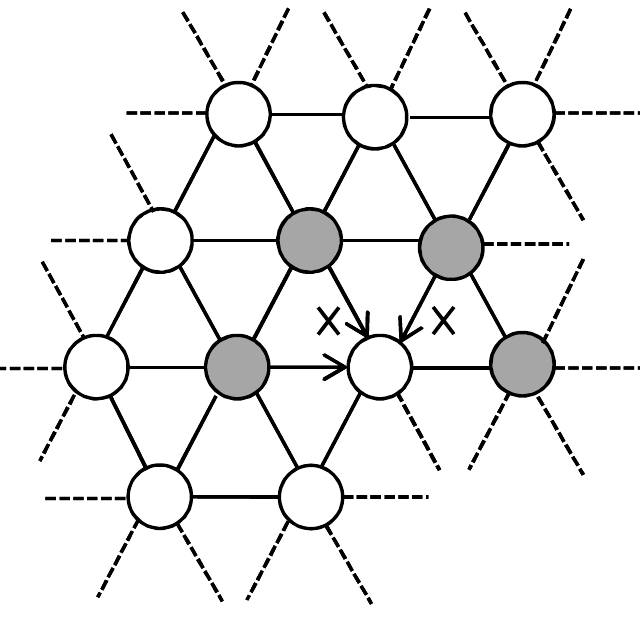}
			\caption{Prohibited behaviors when 
				a robot with one adjacent robot node NE moves to E.}
			\label{fig:case3Counter}
		\end{center}
	\end{minipage} 
	\begin{minipage}{0.01\hsize}
		\hspace{2mm}
	\end{minipage}
	\begin{minipage}{0.47\hsize}
		\begin{center}
			\includegraphics[scale=0.475]{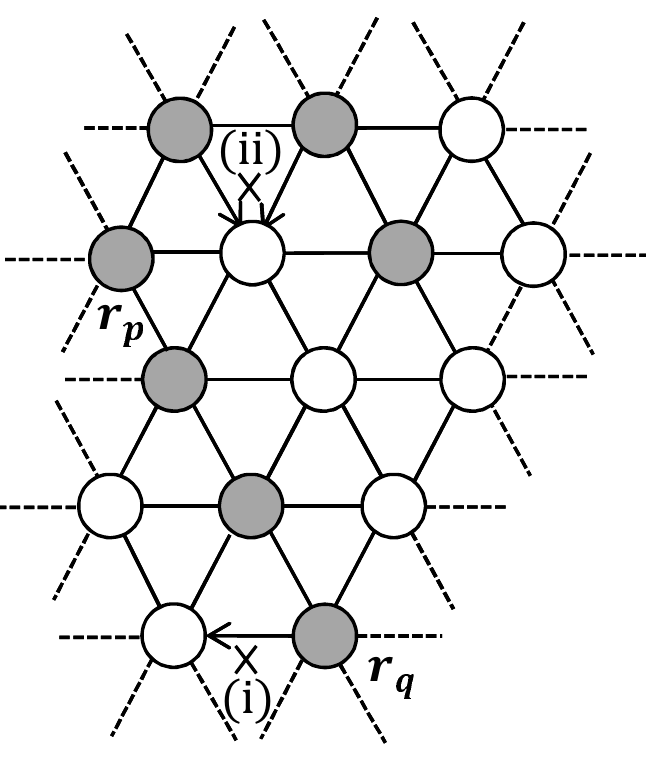}
			\caption{A configuration such that 
				only $r_p$ or $r_p$ can leave the current node 
				((i): by Lemma \ref{lem:noNWtoW},
				(ii): by Fig.\,\ref{fig:case3Counter}).
			}
			\label{fig:case3Mid}
		\end{center}
	\end{minipage}
\end{figure}

\begin{figure}[t!]
	\begin{center}
		\includegraphics[scale=0.59]{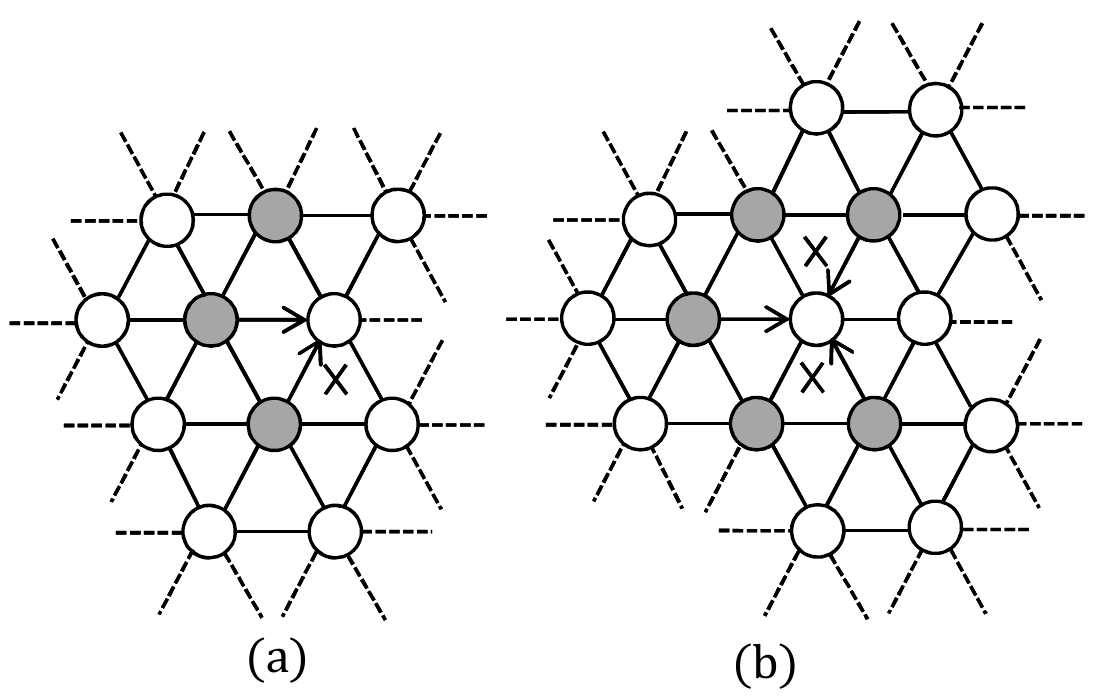}
		\caption{Prohibited behaviors when 
			a robot with two adjacent robot node NE and SE moves to E.}
		\label{fig:case3-1Counter}
	\end{center}
\end{figure}

\begin{figure}[t!]
	\begin{center}
		\includegraphics[scale=0.55]{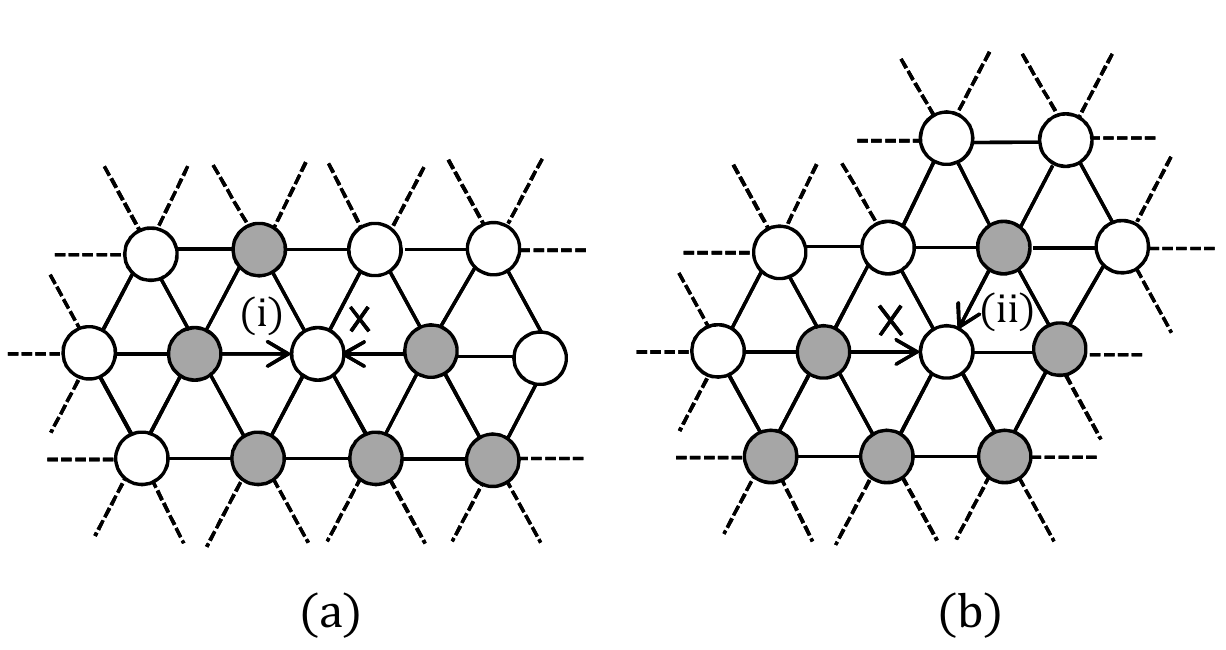}
		\caption{Configurations showing that 
			a robot with two adjacent robot nodes SW and SE must stay at the current node
			((i): assumption of Case 3-1,
			(ii): assumption of Algorithm $\cal{A}$).
		}
		\label{fig:case3-1Counter2}
	\end{center}
\end{figure}
\begin{figure}[t!]
	\begin{center}
		\includegraphics[scale=0.49]{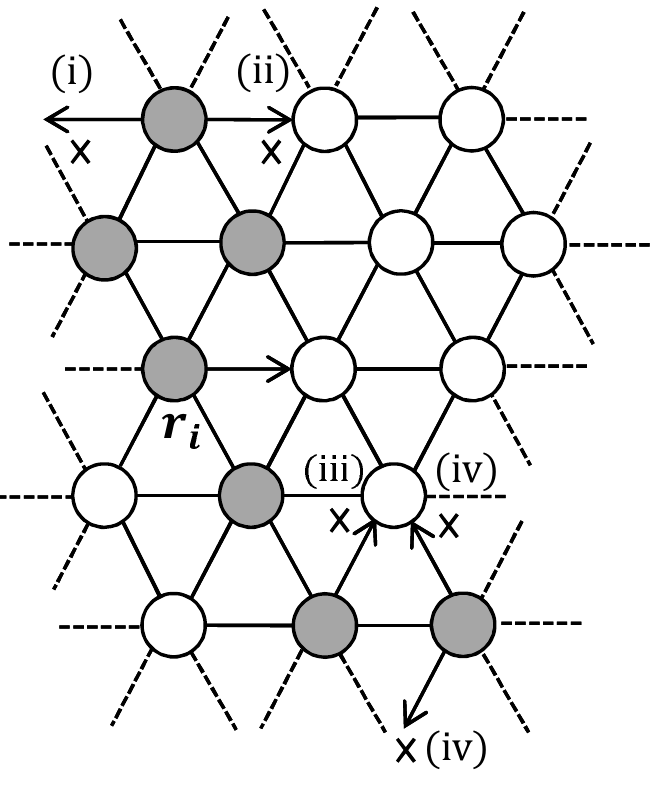}
		\caption{A configuration such that 
			only robot $r_i$ with three adjacent robot nodes NW, NE, and SE can 
			leave the current node
			((i): by Fig.\,\ref{fig:case3-1Counter2} (a),
			(ii): by Fig.\,\ref{fig:case3-1Counter2} (b),
			(iii): by Fig.\,\ref{fig:assumptionCounter} (d),  
			(iv): by Fig.\,\ref{fig:case3-1Counter} (b)).
		}
		\label{fig:case3-1Mid}
	\end{center}
\end{figure}

\begin{figure}[t!]
	\begin{center}
		\includegraphics[scale=0.55]{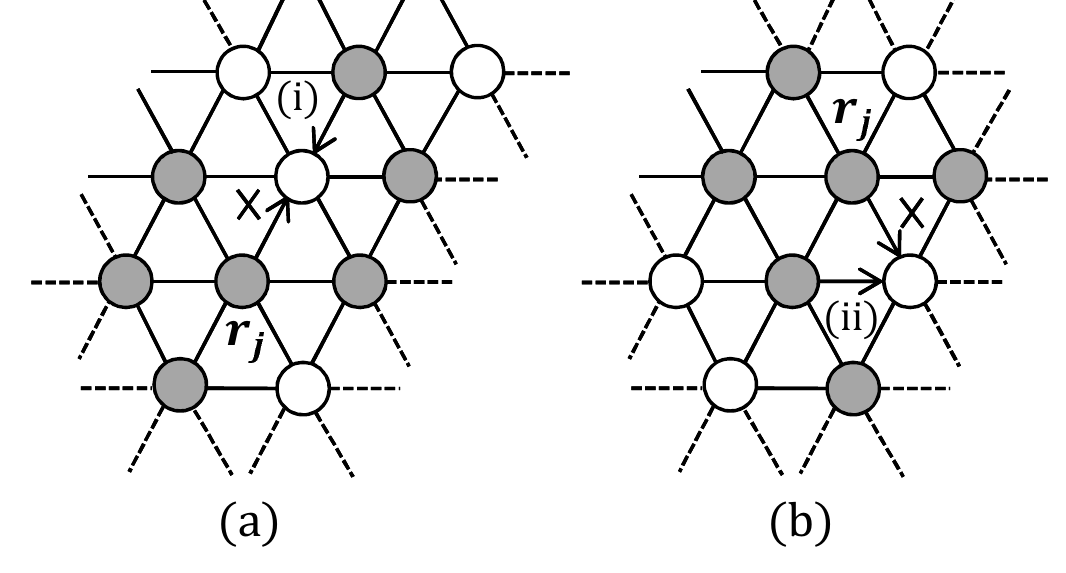}
		\caption{Configurations showing that robot $r_j$ with 
			four adjacent robot nodes E, SW, W, and NW must stay at the current node
			((i): assumption of Algorithm $\cal{A}$,
			(ii): by Fig.\,\ref{fig:case3-1Mid}).}
		\label{fig:case3-1Counter3}
	\end{center}
\end{figure}

\begin{figure}[t!]
	\begin{center}
		\includegraphics[scale=0.4]{{eps/case3-1Mid2}}
		\caption{A configuration such that 
			only robot $r_i$ with two adjacent robot nodes NE and NW can leave the current node
			((i): by Fig.\,\ref{fig:case3-1Counter2} (a),
			(ii): by Fig.\,\ref{fig:case3-1Counter2} (b),
			(iii): by Fig.\,\ref{fig:case3-1Counter3} (a), 
			(iv): by Fig.\,\ref{fig:case3-1Counter3} (b), 
			(v): by Fig.\,\ref{fig:case3-1Counter} (b)).}
		\label{fig:case3-1Mid2}
	\end{center}
\end{figure}

\begin{figure}[t!]
	\begin{center}
		\includegraphics[scale=0.6]{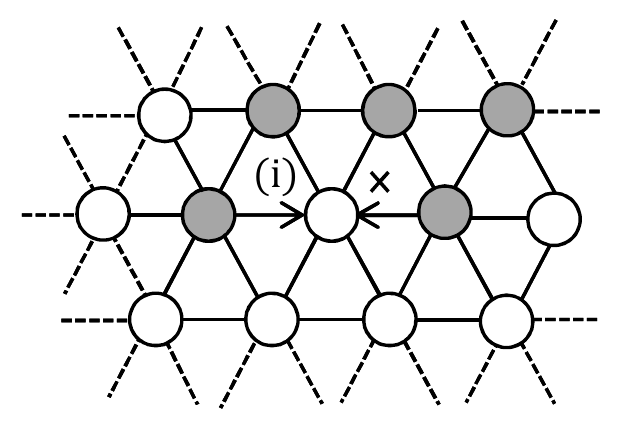}
		\caption{An example showing that a robot with two adjacent robot nodes 
			NE and NW cannot move to W
			((i): assumption of Case 3).}
		\label{fig:case3-1Counter4}
	\end{center}
\end{figure}

\begin{figure}[t!]
	\begin{center}
		\includegraphics[scale=0.65]{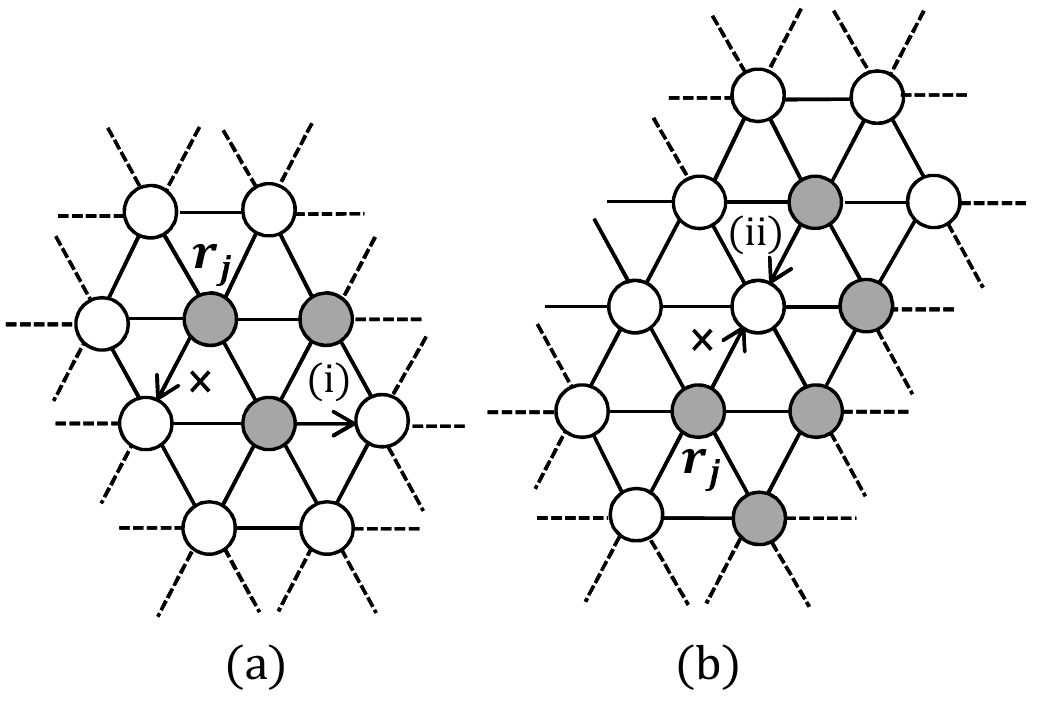}
		\caption{Configurations showing that 
			a robot with two adjacent robot nodes SW and SE must stay at the current node
			((i): by Fig.\,\ref{fig:case3-1Counter4}, 
			(ii): assumption of Algorithm $\cal{A}$).
		}
		\label{fig:case3-1Counter5}
	\end{center}
\end{figure}

\begin{figure}[t!]
	\begin{center}
		\includegraphics[scale=0.45]{{eps/case3-1final}}
		\caption{An unsolvable configuration 
			when 
			a robot with two adjacent robot nodes NE and SE moves to E
			((i): by Fig.\,\ref{fig:case3-1Counter5} (a),
			(ii): by Fig.\,\ref{fig:case3-1Counter5} (b), 
			(iii): by Fig.\,\ref{fig:case3-1Counter} (b)).}
		\label{fig:case3-1final}
	\end{center}
\end{figure}

\textit{Case 3-1: robot $r_p$ moves to E}.
In this case, the following two robot behaviors are not allowed (Fig.\,\ref{fig:case3-1Counter}):
(a) a robot with one adjacent robot node NW moves to NE and 
(b) a robot with one adjacent robot node W moves to NW or SW.
In addition, when considering configurations of Fig.\,\ref{fig:case3-1Counter2},
a robot with two adjacent robot node SW and SE must stay at the current node.
Then, in a configuration of Fig.\,\ref{fig:case3-1Mid}
only robot $r_i$ with three adjacent robot nodes NW, NE, and SE can leave the current node
and it needs to move to E to avoid a collision or an unconnected configuration.
In addition, when considering configurations of Fig.\,\ref{fig:case3-1Counter3},
a robot $r_j$ with four adjacent robots nodes E, SW, W, and NW must stay at the current node.
Hence, in a configuration of Fig.\,\ref{fig:case3-1Mid2},
only  robot $r_i$ with two adjacent robot nodes NW and NE can leave the current node.
Then, 
in a configuration of Fig.\,\ref{fig:case3-1Counter4},
a collision occurs when $r_i$ moves to W.
Hence, it needs to move to E.
Next, we consider the behavior of robot $r_j$ with two adjacent robot nodes E and SE.
When it moves to SW, the configuration becomes unconnected if
its southeast robot $r_i$ has two adjacent robot nodes NW and NE, and moves to E
(Fig.\,\ref{fig:case3-1Counter5} (a)).
In addition, $r_j$ cannot move to NE since a collision occurs in a configuration of Fig.\,\ref{fig:case3-1Counter5} (b). 
Thus, robot $r_j$ must stay at the current node.
Finally, in a configuration of Fig.\,\ref{fig:case3-1final},
no robot can leave the current node and robots cannot achieve gathering.
Therefore, a robot $r_p$ with two adjacent robot nodes NE and SE must stay at the current node.
By a similar discussion, we can also show that a robot with three adjacent robot nodes
NE, SE, and SW must stay at the current node.

\if()

\begin{figure}[t!]
	\begin{minipage}{0.55\hsize}
		\begin{center}
			\includegraphics[scale=0.5]{eps/case3-1Mid}
			\caption{A configuration such that 
				only robot $r_i$ with three adjacent robot nodes NW, NE, and SE can 
				leave the current node
				((i): by Fig.\,\ref{fig:case3-1Counter2} (a),
				(ii): by Fig.\,\ref{fig:case3-1Counter2} (b),
				(iii): by Fig.\,\ref{fig:assumptionCounter} (d),  
				(iv): by Fig.\,\ref{fig:case3-1Counter} (b)).
			}
			\label{fig:case3-1Mid}
		\end{center}
	\end{minipage}
	\begin{minipage}{0.01\hsize}
		\hspace{2mm}
	\end{minipage}
	\begin{minipage}{0.4\hsize}
		\begin{center}
			\includegraphics[scale=0.55]{eps/case3-1Counter3}
			\caption{Configurations showing that robot $r_j$ with 
				four adjacent robot nodes E, SW, W, and NW must stay at the current node
				((i): assumption of Algorithm $\cal{A}$,
				(ii): by Fig.\,\ref{fig:case3-1Mid}).}
			\label{fig:case3-1Counter3}
		\end{center}
	\end{minipage}
\end{figure}

\fi

\if()
\begin{figure}[t!]
	\begin{minipage}{0.48\hsize}
		\begin{center}
			\includegraphics[scale=0.5]{{eps/case3-1Mid2}}
			\caption{A configuration such that 
				only robot $r_i$ with two adjacent robot nodes NE and NW can leave the current node
				((i): by Fig.\,\ref{fig:case3-1Counter2} (a),
				(ii): by Fig.\,\ref{fig:case3-1Counter2} (b),
				(iii): by Fig.\,\ref{fig:case3-1Counter3} (a), 
				(iv): by Fig.\,\ref{fig:case3-1Counter3} (b), 
				(v): by Fig.\,\ref{fig:case3-1Counter} (b)).
			}
			\label{fig:case3-1Mid2}
		\end{center}
	\end{minipage}
	\begin{minipage}{0.01\hsize}
		\hspace{2mm}
	\end{minipage}
	\begin{minipage}{0.48\hsize}
		\begin{center}
			\includegraphics[scale=0.65]{eps/case3-1Counter4}
			\caption{An example showing that a robot with two adjacent robot nodes 
				NE and NW cannot move to W
				((i): assumption of Case 3).}
			\label{fig:case3-1Counter4}
		\end{center}
	\end{minipage}
\end{figure}
\fi

\if()

\begin{figure}[t!]
	\begin{minipage}{0.45\hsize}
		\begin{center}
			\includegraphics[scale=0.55]{eps/case3-1Counter5}
			\caption{Configurations showing that 
				a robot with two adjacent robot nodes SW and SE must stay at the current node
				((i): by Fig.\,\ref{fig:case3-1Counter4}, 
				(ii): assumption of Algorithm $\cal{A}$).
			}
			\label{fig:case3-1Counter5}
		\end{center}
	\end{minipage}
	\begin{minipage}{0.01\hsize}
		\hspace{2mm}
	\end{minipage}
	\begin{minipage}{0.51\hsize}
		\begin{center}
			\includegraphics[scale=0.45]{{eps/case3-1final}}
			\caption{An unsolvable configuration 
				when
				a robot with two adjacent robot nodes NE and SE moves to E
				((i): by Fig.\,\ref{fig:case3-1Counter5} (a),
				(ii): by Fig.\,\ref{fig:case3-1Counter5} (b), 
				(iii): by Fig.\,\ref{fig:case3-1Counter} (b)).}
			\label{fig:case3-1final}
		\end{center}
	\end{minipage}
\end{figure}

\fi

\begin{figure}[t!]
	\begin{minipage}{0.48\hsize}
		\begin{center}
			\includegraphics[scale=0.55]{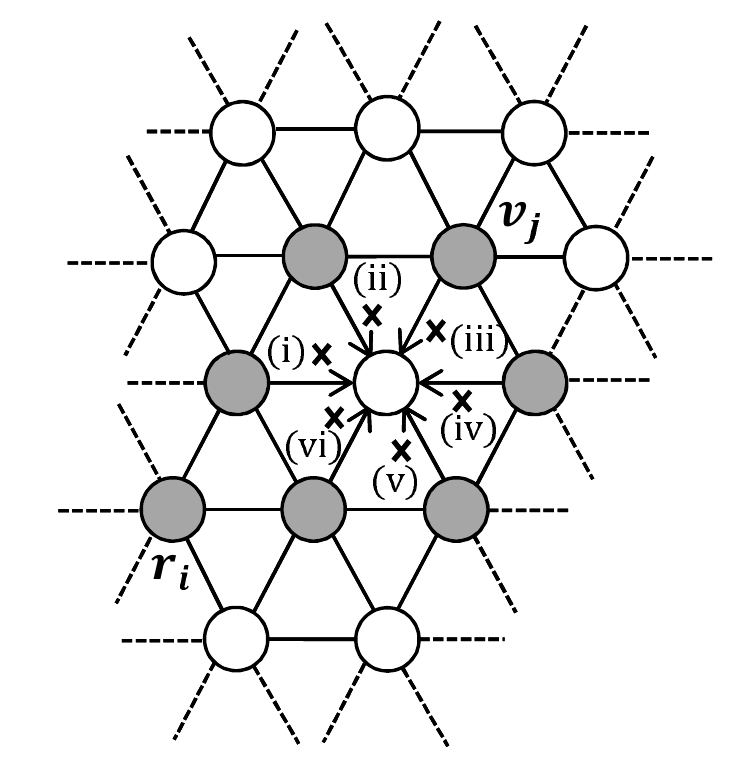}
			\caption{A configuration such that only robot $r_i$ with two adjacent robot nodes 
				NE and E can leave the current node
				((i): by Case 3-1,
				(ii), (iii): by Fig.\,\ref{fig:case3Counter},  
				(iv)-(vi): behaviors that may cause a collision if 
				a robot at $v_j$ has one adjacent robot node SE and 
				moves to SW by the hypothesis of Algorithm $\cal{A}$).}
			\label{fig:case3-2counter1}
		\end{center}
	\end{minipage}
	\begin{minipage}{0.01\hsize}
		\hspace{2mm}
	\end{minipage}
	\begin{minipage}{0.48\hsize}
		\begin{center}
			\includegraphics[scale=0.6]{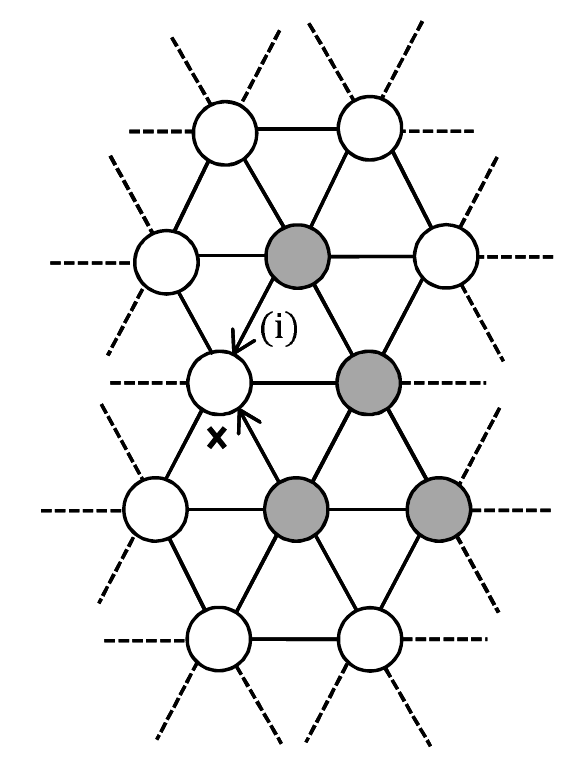}
			\caption{An example showing that 
				a robot with two adjacent robot nodes NE and E cannot move to NW
				((i): assumption of Algorithm $\cal{A}$).}
			\label{fig:case3-2counter2}
		\end{center}
	\end{minipage}
\end{figure}

\begin{figure}[t!]
	\begin{minipage}{0.48\hsize}
		\begin{center}
			\includegraphics[scale=0.55]{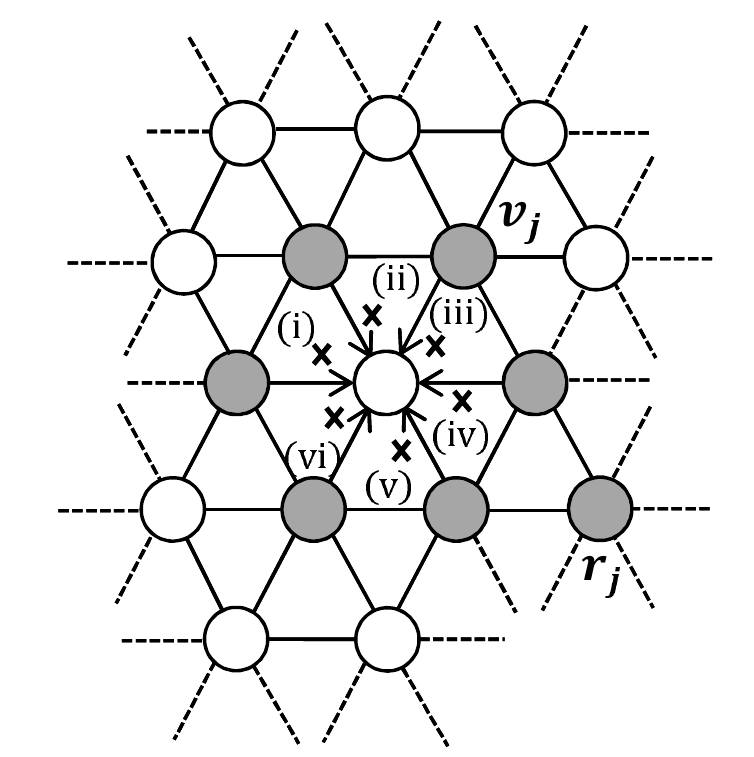}
			\caption{A configuration such that only robot $r_i$ with two adjacent robot nodes 
				NW and W can leave the current node
				((i): by Case 3-1,
				(ii), (iii): by Fig.\,\ref{fig:case3Counter},  
				(iv)-(vi): behaviors that may cause a collision if 
				a robot at $v_j$ has one adjacent robot node SE and 
				moves to SW by the hypothesis of Algorithm $\cal{A}$).}
			\label{fig:case3-2counter3}
		\end{center}
	\end{minipage}
	\begin{minipage}{0.01\hsize}
		\hspace{2mm}
	\end{minipage}
	\begin{minipage}{0.48\hsize}
		\begin{center}
			\includegraphics[scale=0.55]{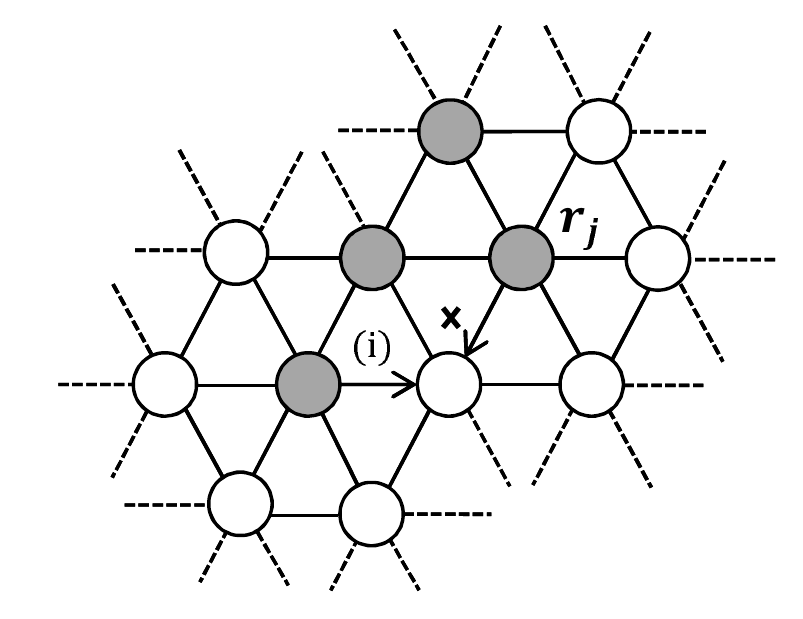}
			\caption{An example showing that 
				a robot with two adjacent robot nodes NW and W cannot move to SW
				((i): assumption of Case 3).}
			\label{fig:case3-2counter4}
		\end{center}
	\end{minipage}
\end{figure}

\textit{Case 3-2: robot $r_q$ moves to NE.}	
In this case, when considering a configuration of Fig.\,\ref{fig:case3-2counter1},
only robot $r_i$ with two adjacent robot nodes E and NE can leave the current node.
However, $r_i$ cannot move to NW since a collision occurs in a configuration of 
Fig.\,\ref{fig:case3-2counter2}
and it needs to move to SE.
Similarly, in a configuration of Fig.\, \ref{fig:case3-2counter3}
only robot $r_j$ with two adjacent robot nodes NW and W can leave the current node.
However, $r_j$ cannot move to SW since a collision occurs in a configuration of Fig.\,\ref{fig:case3-2counter4} 
and it needs to move to NE.
Now, we consider the configuration of Fig.\,\ref{fig:case3-2final}.
In the figure, robot $r_i$ moves to SE and robot $r_j$ moves to NE by the above discussion.
Then, the configuration become unconnected and 
robots cannot achieve gathering. 
Thus, a robot $r_q$ with one adjacent robot node NW must stay at the current node.
Therefore, robots cannot solve the problem from the configuration of Fig.\,\ref{fig:case3Mid}
and we have the following lemma.


\begin{lemma}
	\label{lem:noNEtoE}
	A robot with one adjacent robot node \textit{NE} cannot move to node E. 
\end{lemma}

\textbf{Case 4: a robot with one adjacent robot node NW moves to NE.} 
In this case, the following two robot behaviors are not allowed (Fig.\,\ref{fig:case4Counter}):
a robot with two adjacent robot nodes SW and E 
(resp., W and SE) moves to SE (resp., SW).
Then, in a configuration of Fig.\,\ref{fig:case4},
no robot can leave the current node and 
robots cannot achieve gathering.
Thus, we have the following lemma.


\begin{lemma}
	\label{lem:noNWtoNE}
	A robot with one adjacent robot node \textit{NW} cannot move to node \textit{NE}. 
\end{lemma}

\begin{figure}[t!]
	\begin{center}
		\includegraphics[scale=0.55]{{eps/case3-2final}}
		\caption{A configuration that becomes unconnected after 
			robots $r_i$ and $r_j$ move.}
		\label{fig:case3-2final}
	\end{center}
\end{figure}

\begin{figure}[t!]
	\begin{minipage}{0.48\hsize}
		\begin{center}
			\includegraphics[scale=0.55]{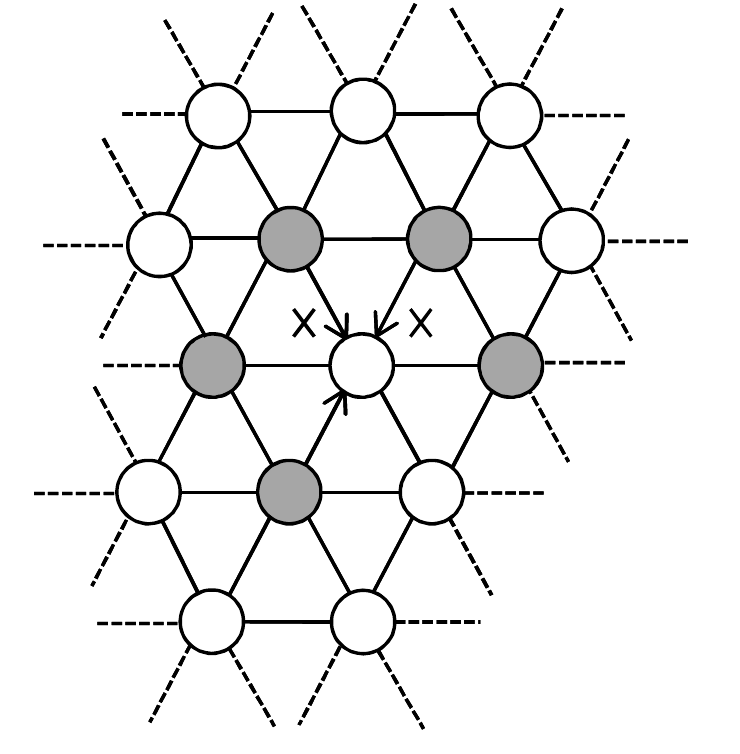}
			\caption{Prohibited behaviors when 
				a robot with one adjacent robot node NW moves to NE.}
			\label{fig:case4Counter}
		\end{center}
	\end{minipage}
	\begin{minipage}{0.01\hsize}
		\hspace{2mm}
	\end{minipage}
	\begin{minipage}{0.48\hsize}
		\begin{center}
			\includegraphics[scale=0.55]{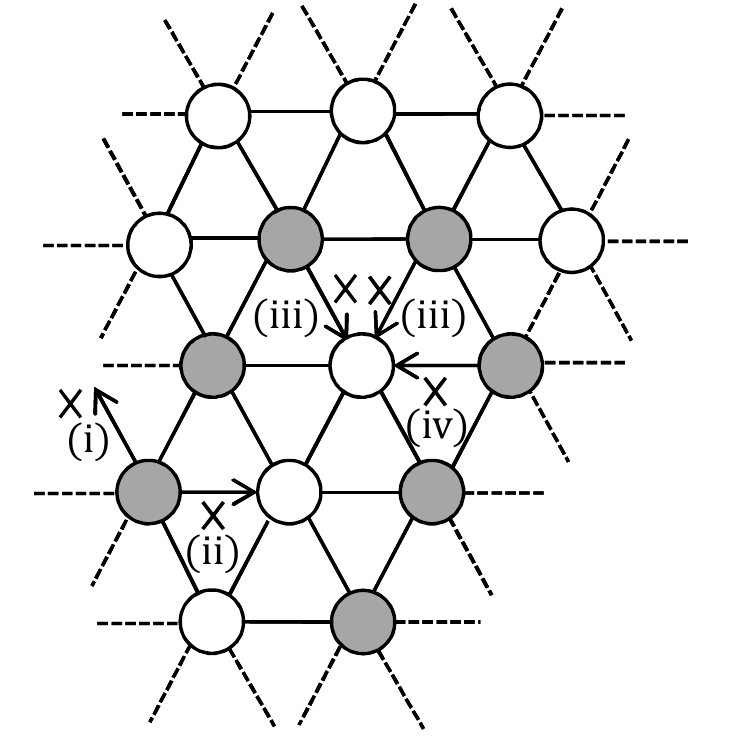}
			\caption{A configuration such that no robot can leave the current node when 
				a robot with one adjacent robot node NW moves to NE
				((i): by Fig.\,\ref{fig:assumptionCounter} (a), 
				(ii): by Lemma \ref{lem:noNEtoE},
				(iii): by Fig.\,\ref{fig:case4Counter}, 
				(iv): by Fig.\,\ref{fig:assumptionCounter} (b)).}
			\label{fig:case4}
		\end{center}
	\end{minipage}
\end{figure}

\begin{figure}[t!]
	\begin{center}
		\includegraphics[scale=0.55]{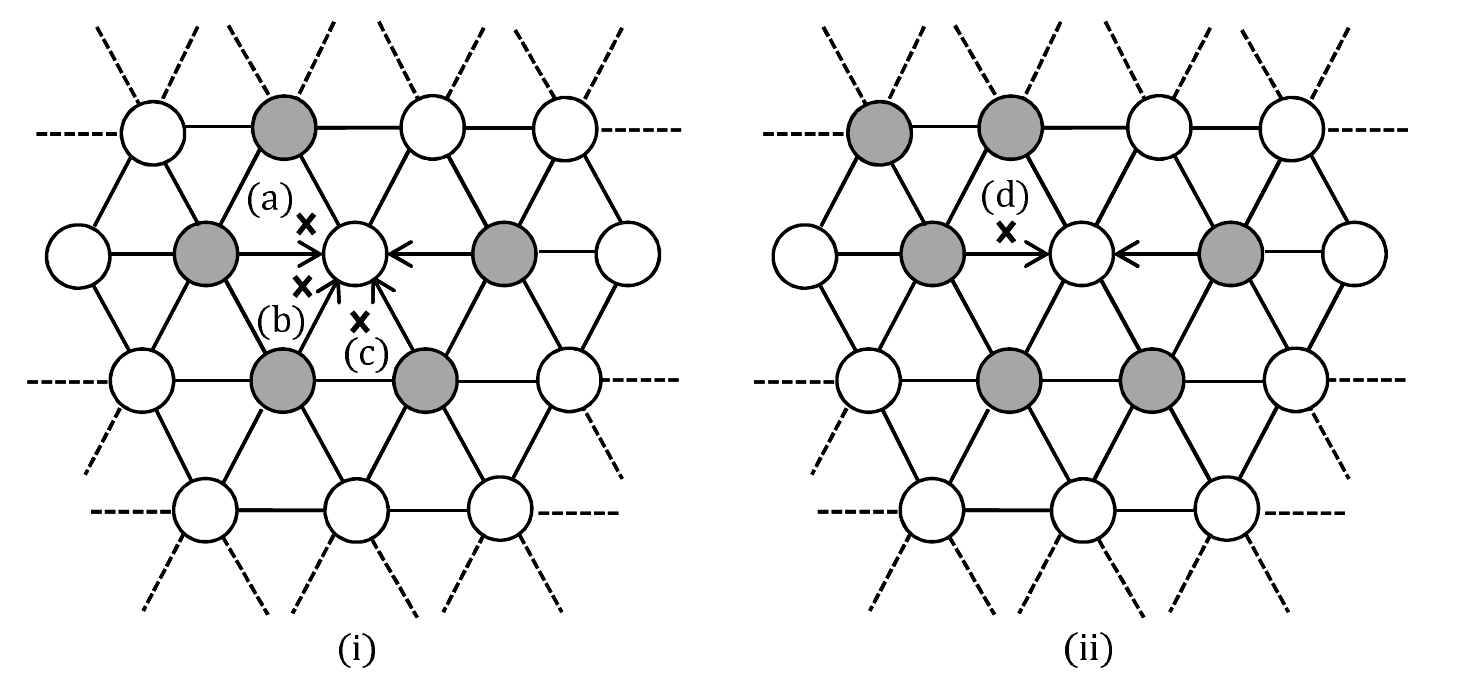}
		\caption{Prohibited behaviors when a robot with one adjacent robot node SW moves to W.}
		\label{fig:case5Counter}
	\end{center}
\end{figure}

\textbf{Case 5: a robot with one adjacent robot node SW moves to W.} 
In this case, the following four robot behaviors are not allowed (Fig.\,\ref{fig:case5Counter}):
(a): a robot with two adjacent robot nodes NE and SE moves to E,
(b): a robot with two adjacent robot nodes NW and E moves to NE,
(c): a robot with two adjacent robot nodes W and NE moves to NW, and 
(d): a robot with three adjacent robot nodes NE, NW, and SE moves to E.
Then, in a configuration of Fig.\,\ref{fig:case5Mid},
it is necessary that 
robot $r_p$ with two adjacent robot nodes SW and E moves to SE or 
robot $r_q$ with two adjacent robot nodes W and SE moves to SW.
We consider each of the behaviors and show for both the cases that 
robot cannot achieve gathering from some configuration.

\begin{figure}[t!]
	\begin{center}
		\includegraphics[scale=0.475]{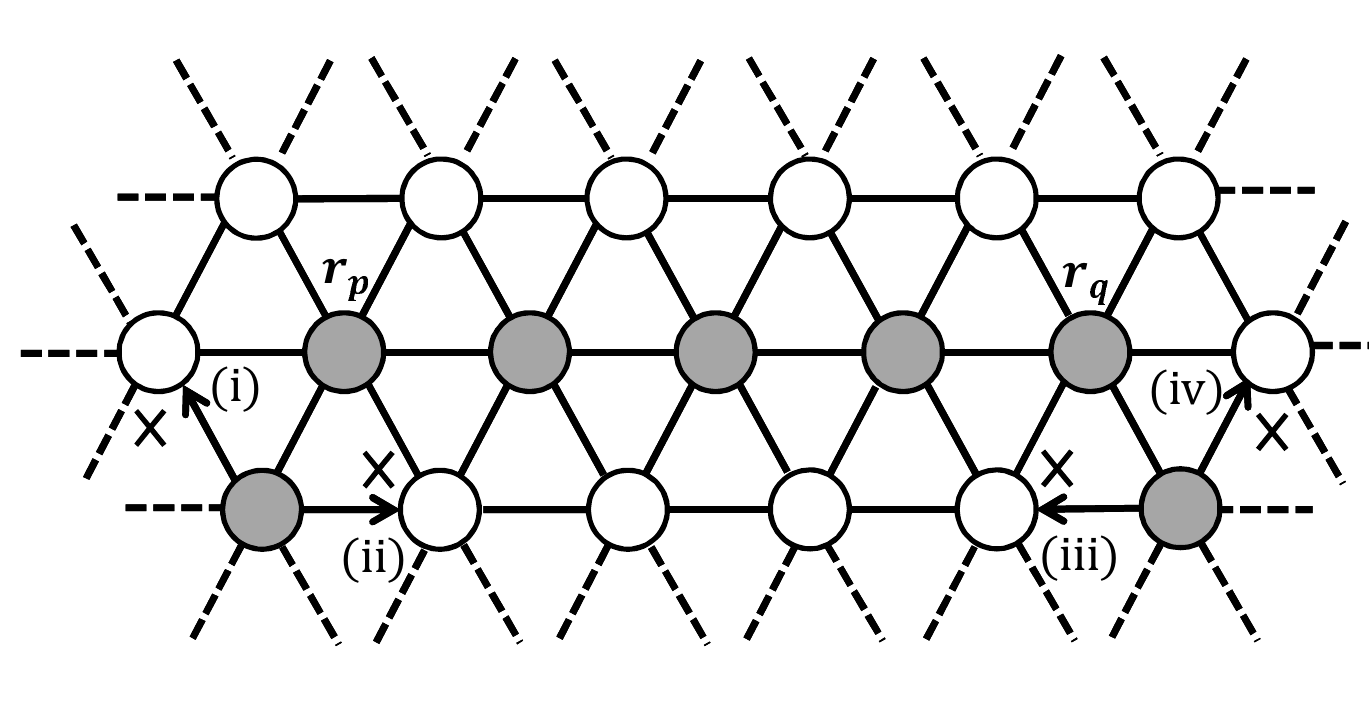}
		\caption{A  configuration such that 
			only robot $r_p$ with two adjacent robot nodes SW and E or 
			$r_q$ with two adjacent robot nodes W and SE can leave the current node				
			((i): by Fig \ref{fig:assumptionCounter} (a), 
			(ii): by Lemma \ref{lem:noNEtoE},
			(iii): by Lemma \ref{lem:noNWtoW}, and 
			(iv): by Lemma \ref{lem:noNWtoNE}).}
		\label{fig:case5Mid}
	\end{center}
\end{figure}

\begin{figure}[t!]
	\begin{center}
		\includegraphics[scale=0.55]{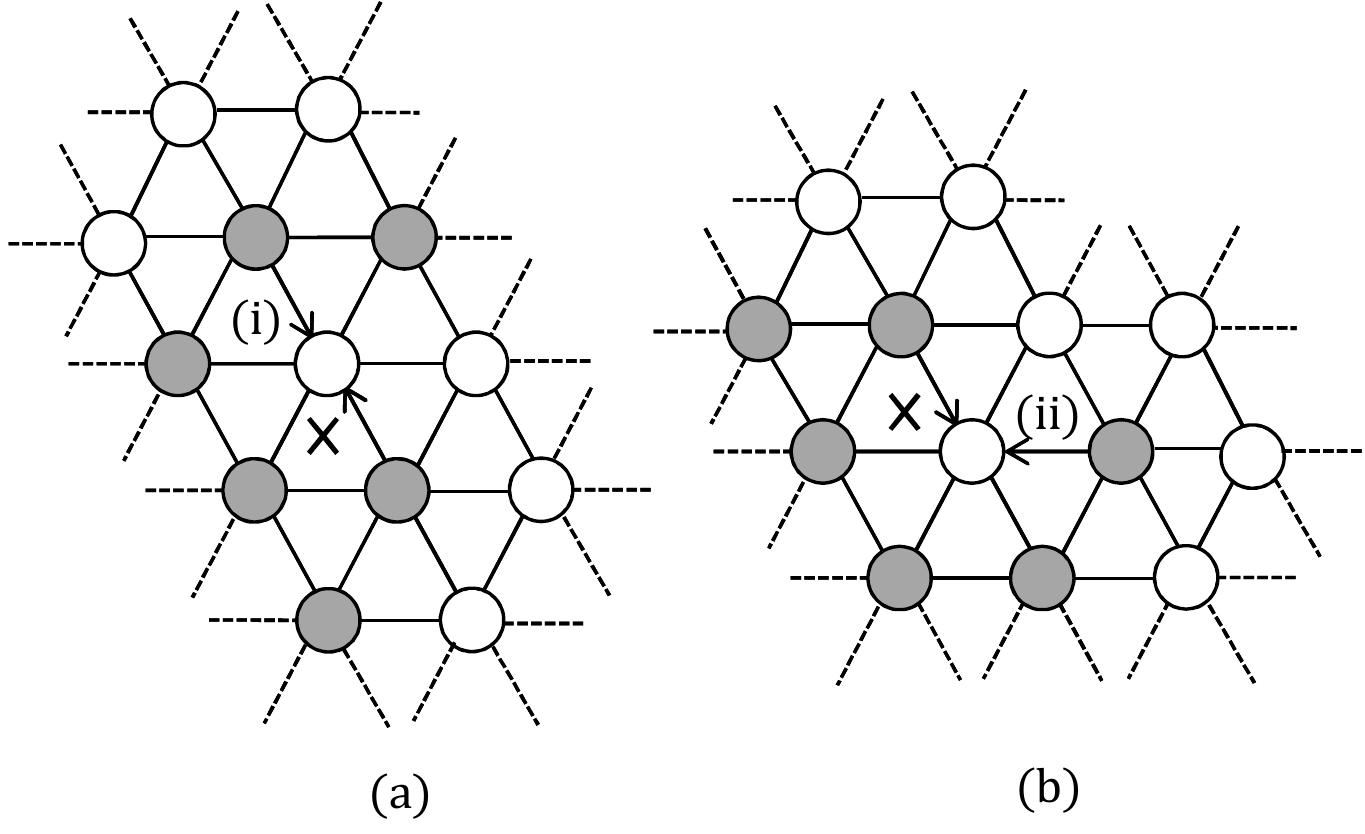}
		\caption{Examples showing that a robot with two adjacent robot nodes W and SW cannot leave the current node 
			((i): assumption of Case 5-1,
			(ii): assumption of Case 5).}
		\label{fig:case5-1Counter}
	\end{center}
\end{figure}

\if()

\begin{figure}[t!]
	\begin{minipage}{0.48\hsize}
		\begin{center}
			\includegraphics[scale=0.55]{eps/case5Counter}
			\caption{Prohibited behaviors 
				a robot with one adjacent robot node SW moves to W.}
			\label{fig:case5Counter}
		\end{center}
	\end{minipage}
	\begin{minipage}{0.01\hsize}
		\hspace{2mm}
	\end{minipage}
	\begin{minipage}{0.48\hsize}
		\begin{center}
			\includegraphics[scale=0.5]{eps/case5Mid}
			\caption{A  configuration such that 
				only robot $r_p$ with two adjacent robot nodes SW and E or 
				$r_q$ with two adjacent robot nodes W and SE can leave the current node				
				((i): by Fig \ref{fig:assumptionCounter} (a), 
				(ii): by Lemma \ref{lem:noNEtoE},
				(iii): by Lemma \ref{lem:noNWtoW}, and 
				(iv): by Lemma \ref{lem:noNWtoNE}).}
			\label{fig:case5Mid}
		\end{center}
	\end{minipage}
\end{figure}

\fi

\textit{Case 5-1: robot $r_q$ moves to SE.}
In this case, a robot with two adjacent robot nodes  SW and W cannot leave the current node
because otherwise a collision may occur (Fig.\,\ref{fig:case5-1Counter}).
Next, when considering a configuration of Fig.\,\ref{fig:case5-1Mid},
only robot $r_i$ with two adjacent robot nodes SE and SW can leave the current node.
However, in a configuration of Fig.\,\ref{fig:case5-1Counter2},
$r_i$ cannot move to E since a collision occurs.
Hence, $r_i$ needs to move to W.
Then, a robot $r_j$ with two adjacent robot nodes E and NE
cannot move to NW or SE because a collision occurs or the configuration become unconnected
(Fig.\,\ref{fig:case5-1Counter3}).
Finally, we consider the configuration of Fig.\,\ref{fig:case5-1final}.
In the figure, no robot can leave the current node and they cannot achieve gathering.
Thus, a robot $r_p$ with two adjacent robot nodes SW and E cannot move to SE 
and it must stay at the current node.


\begin{figure}[t!]
	\begin{minipage}{0.49\hsize}
		\begin{center}
			\includegraphics[scale=0.55]{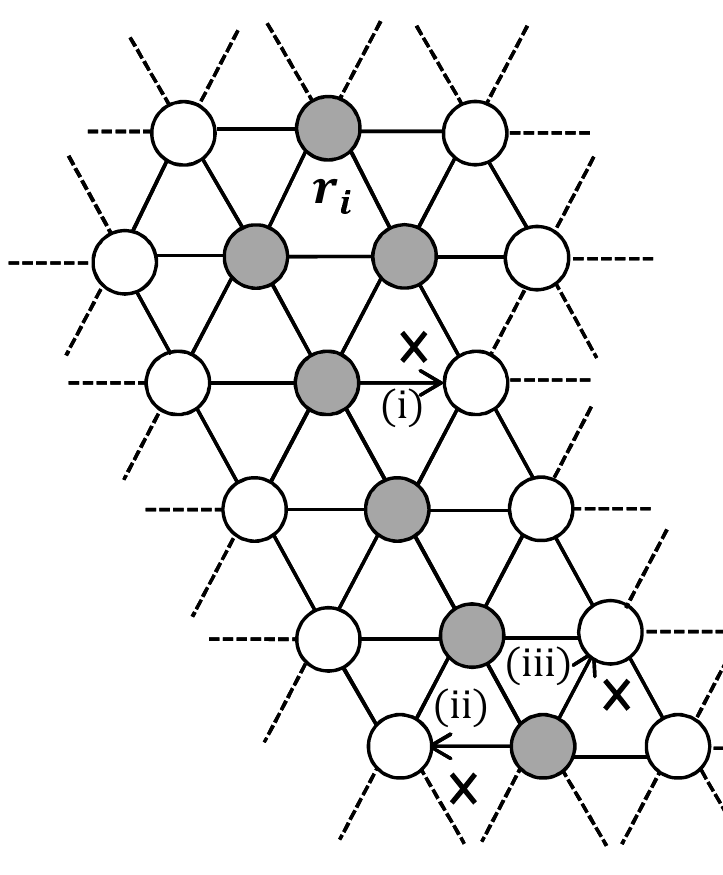}
			\caption{A configuration such that only robot $r_i$ with two adjacent robot nodes
				SE and SW can leave the current node
				((i): by Fig \ref{fig:case5Counter} (d), 
				(ii): by Lemma \ref{lem:noNWtoW},  
				(iii): by Lemma \ref{lem:noNWtoNE}).}
			\label{fig:case5-1Mid}
		\end{center}
	\end{minipage}
	\begin{minipage}{0.01\hsize}
		\hspace{2mm}
	\end{minipage}
	\begin{minipage}{0.47\hsize}
		\begin{center}
			\includegraphics[scale=0.55]{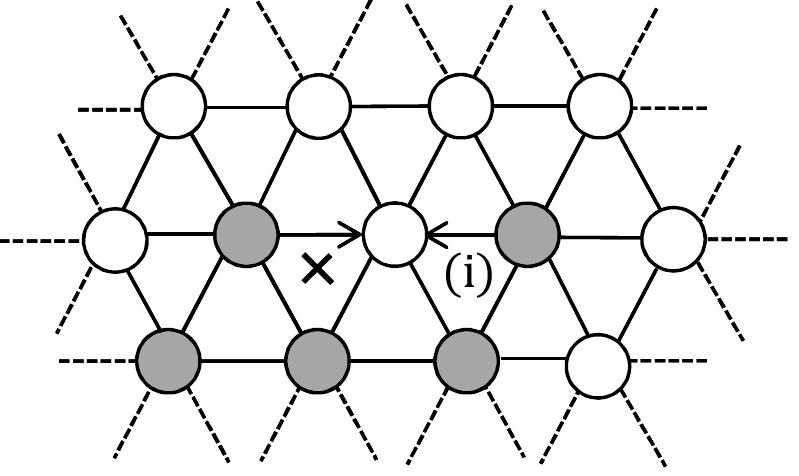}
			\caption{An example showing that a robot with two adjacent robot nodes 
				SE and SW cannot move to E. 
				((i): assumption of Case 5).}
			\label{fig:case5-1Counter2}
		\end{center}
	\end{minipage}
\end{figure}

\begin{figure}[t!]
	\begin{minipage}{0.48\hsize}
		\begin{center}
			\includegraphics[scale=0.55]{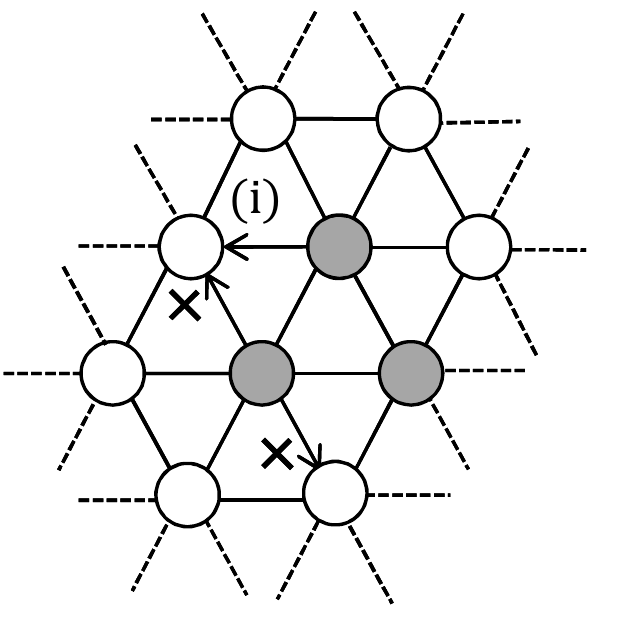}
			\caption{An example showing that a robot with two adjacent robot nodes E and NE
				cannot leave the current node
				((i): by Fig.\,\ref{fig:case5-1Counter2}).
			}
			\label{fig:case5-1Counter3}
		\end{center}
	\end{minipage}
	\begin{minipage}{0.01\hsize}
		\hspace{2mm}
	\end{minipage}
	\begin{minipage}{0.48\hsize}
		\begin{center}
			
			\includegraphics[scale=0.5]{{eps/case5-1final}}
			\caption{An unsolvable configuration when 
				a robot with two adjacent robot nodes SW and E moves to SE
				((i): by Fig.\,\ref{fig:case5-1Counter3},
				(ii): by Fig.\,\ref{fig:case5-1Counter} (a),
				(iii): by Fig.\,\ref{fig:case5-1Counter} (b),
				(iv): by Lemma \ref{lem:noNWtoW},  
				(v): by Lemma \ref{lem:noNWtoNE}).}
			\label{fig:case5-1final}
		\end{center}
	\end{minipage}
\end{figure}

\begin{figure}[t!]
	\begin{center}
		\includegraphics[scale=0.55]{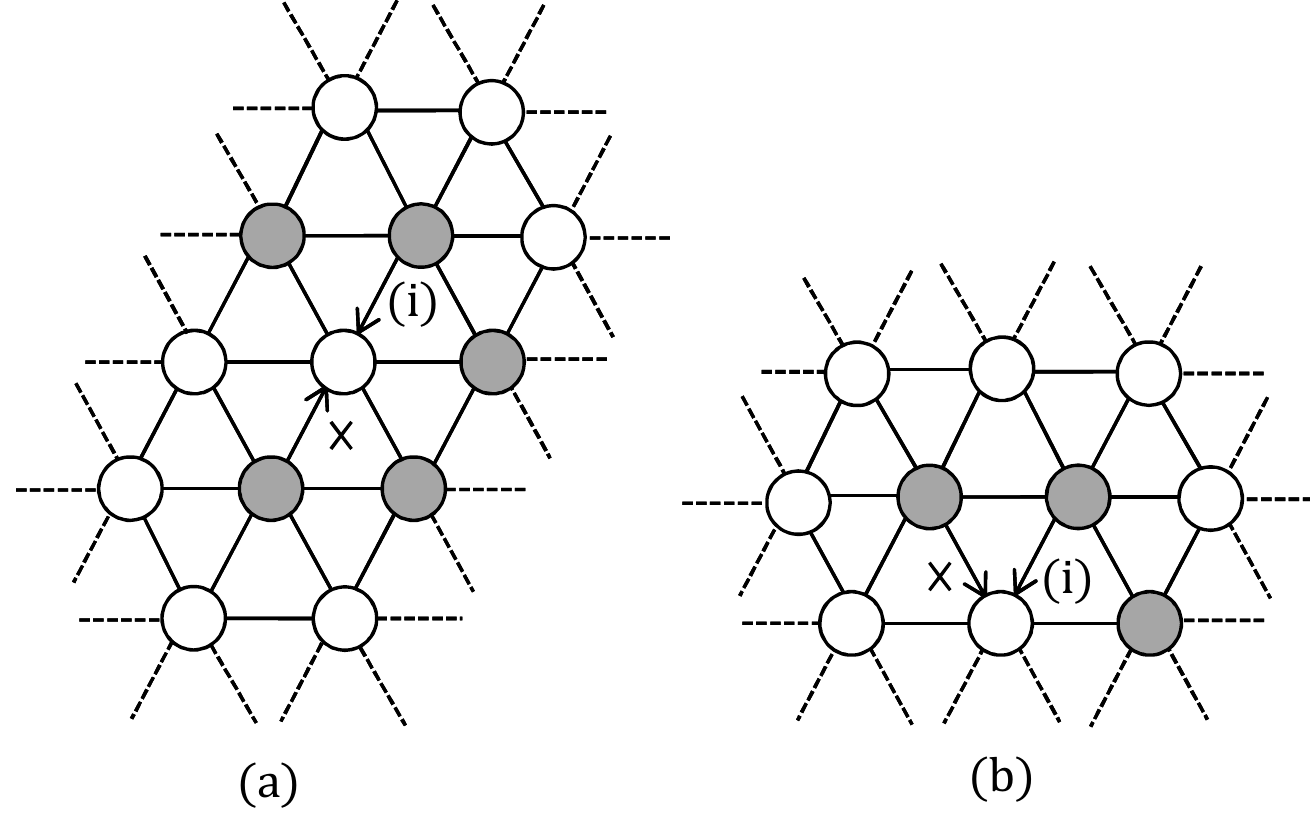}
		\caption{Examples showing that a robot with two adjacent robot nodes E cannot leave the current node
			((i): assumption of Case 5-2).}
		\label{fig:case5-2counter}
	\end{center}
\end{figure}

\if()
\begin{figure}[t!]
	\begin{center}
		\includegraphics[scale=0.5]{{eps/case5-1final}}
		\caption{An unsolvable configuration when
			a robot with two adjacent robot nodes SW and E moves to SE
			((i): by Fig.\,\ref{fig:case5-1Counter3},
			(ii): by Fig.\,\ref{fig:case5-1Counter} (a),
			(iii): by Fig.\,\ref{fig:case5-1Counter} (b),
			(iv): by Lemma \ref{lem:noNWtoW},  
			(v): by Lemma \ref{lem:noNWtoNE}).}
		\label{fig:case5-1final}
	\end{center}
\end{figure}
\fi

\begin{figure}[t!]
\begin{center}
	\includegraphics[scale=0.45]{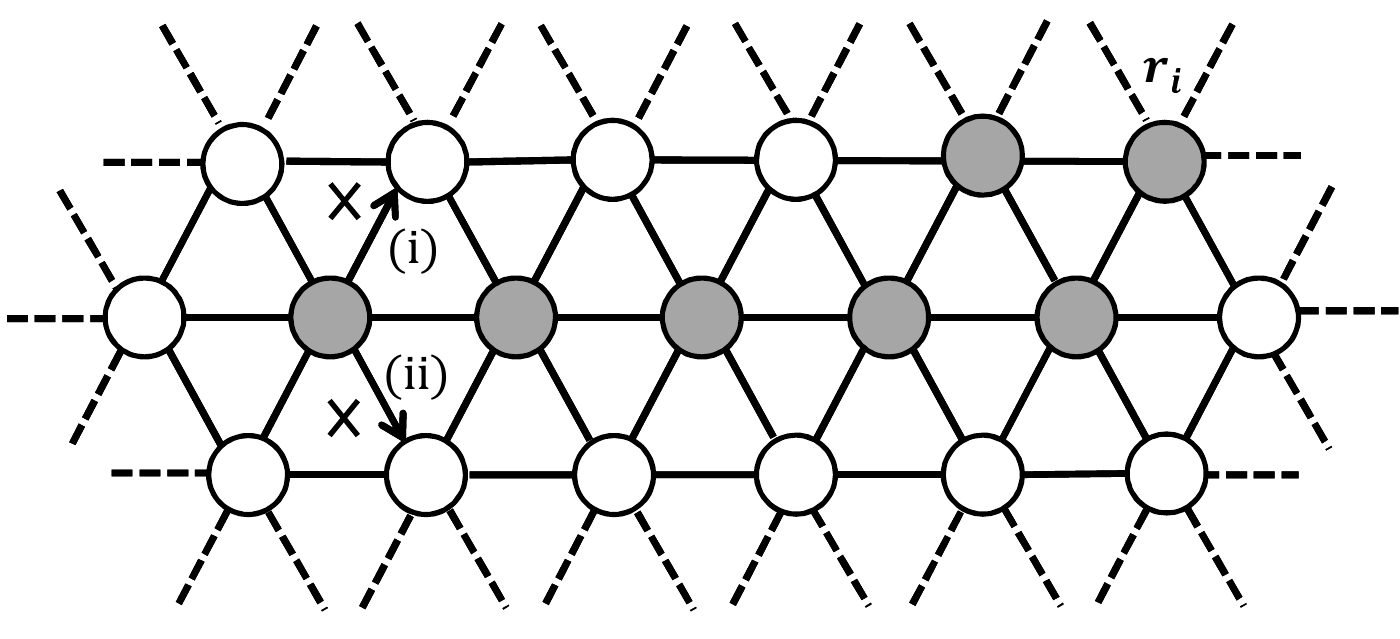}
	\caption{A configuration such only robot $r_i$ with two adjacent robot nodes W and SW
		can leave the current node 
		((i): by Fig.\,\ref{fig:case5-2counter} (a),
		(ii): by Fig.\,\ref{fig:case5-2counter} (b)).}
	\label{fig:case5-2Mid}
\end{center}
\end{figure}

\begin{figure}[t!]
\begin{minipage}{0.48\hsize}
	\begin{center}
		\includegraphics[scale=0.525]{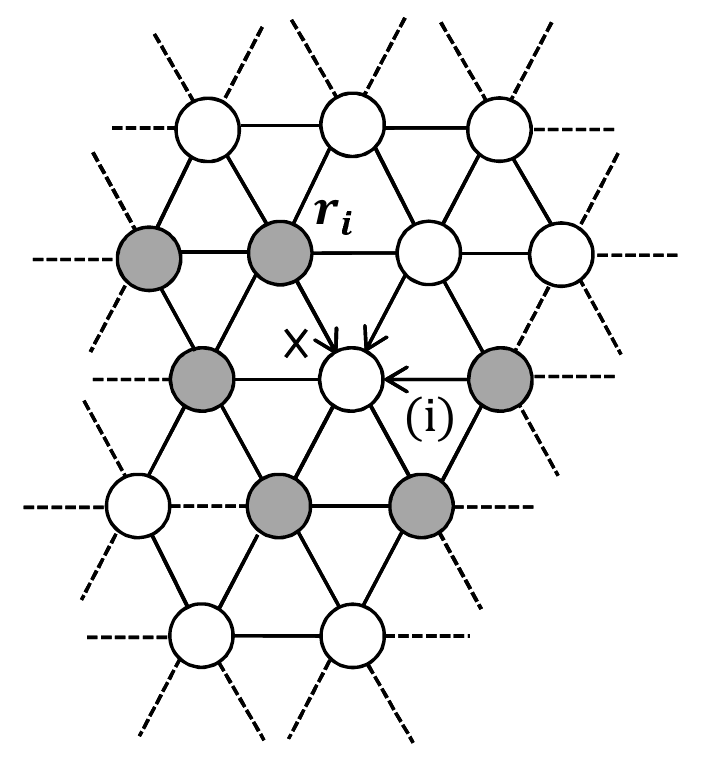}
		\caption{An example showing that a robot with two adjacent robot nodes SW and W
			cannot move to SE
			((i): assumption of Case 5).}
		\label{fig:case5-2counter2}
	\end{center}
\end{minipage}
\begin{minipage}{0.01\hsize}
	\hspace{2mm}
\end{minipage}
\begin{minipage}{0.48\hsize}
	\begin{center}
		\includegraphics[scale=0.525]{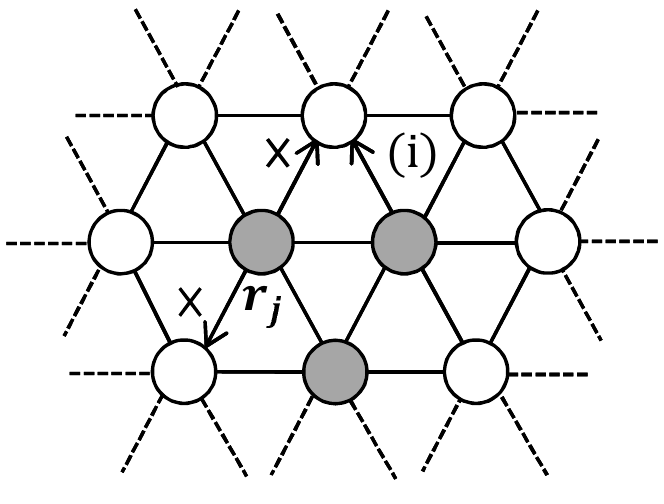}
		\caption{An example showing that a robot with two adjacent robot nodes E and SE 
			cannot leave the current node
			((i): by Fig.\,\ref{fig:case5-2counter2}).}
		\label{fig:case5-2counter3}
	\end{center}
\end{minipage}
\end{figure}

\textit{Case 5-2: robot $r_q$ moves to SW.}
In this case, a robot with one adjacent robot node E cannot leave the current node 
because otherwise a collision may occur (Fig.\,\ref{fig:case5-2counter}).
Then, when considering a configuration of Fig.\,\ref{fig:case5-2Mid},
only robot $r_i$ with two adjacent robot nodes SW and W can leave the current node.
However, $r_i$ cannot move to SE since a collision occurs in a configuration of Fig.\,\ref{fig:case5-2counter2}.
Hence, $r_i$ needs to move to NW.
Then, robot $r_j$ with two adjacent robot nodes E and SE cannot leave the current node 
because a collision occurs or the  configuration becomes unconnected (Fig.\,\ref{fig:case5-2counter3}).
Next, we consider the configuration of Fig.\,\ref{fig:case5-2Mid2}.
In the figure, only robot $r_i$ with three adjacent robot nodes SE, W, and NW can leave the current node 
and it needs to move to SW by the previous discussions.
Then, 
a robot $r_j$ with four adjacent robot nodes E, SW, NW, and NE cannot leave the current node
since a collision occurs in a configuration of Fig.\,\ref{fig:case5-2counter4}.
Thus, when considering  the  configuration of Fig.\,\ref{fig:case5-2Mid3},
only robot $r_i$ with two adjacent robot nodes W and NW can leave the current node.
However, $r_i$ cannot move to NE since a collision occurs in a configuration of Fig.\,\ref{fig:case5-2counter5}.
Thus, $r_i$ needs to move to SW.
Then, a robot with one adjacent robot node W cannot leave the current node 
because otherwise  a collision occurs in configurations of Fig.\,\ref{fig:case5-2counter6}.
Finally, when considering the configuration of Fig.\,\ref{fig:case5-2final},
no robot can leave the current node and robots cannot achieve gathering, which is a contradiction.
Thus, a robot  $r_q$ with two adjacent robot nodes W and SE cannot move to SW.
Therefore, we have the following lemma. 

\if()
\begin{figure}[t!]
	\begin{minipage}{0.48\hsize}
		\begin{center}
			\includegraphics[scale=0.55]{eps/case5-2counter1}
			\caption{Examples showing that a robot with two adjacent robot nodes E cannot leave the current node
				((i): assumption of Case 5-2).}
			\label{fig:case5-2counter}
		\end{center}
	\end{minipage}
	\begin{minipage}{0.01\hsize}
		\hspace{2mm}
	\end{minipage}
	\begin{minipage}{0.48\hsize}
		\begin{center}
			\includegraphics[scale=0.55]{eps/case5-2Mid}
			\caption{A configuration such only robot $r_i$ with two adjacent robot nodes W and SW
				can leave the current node 
				((i): by Fig.\,\ref{fig:case5-2counter} (a),
				(ii): by Fig.\,\ref{fig:case5-2counter} (b)).}
			\label{fig:case5-2Mid}
		\end{center}
	\end{minipage}
\end{figure} 
\fi

\begin{figure}[t!]
	\begin{minipage}{0.48\hsize}
		\begin{center}
			\includegraphics[scale=0.45]{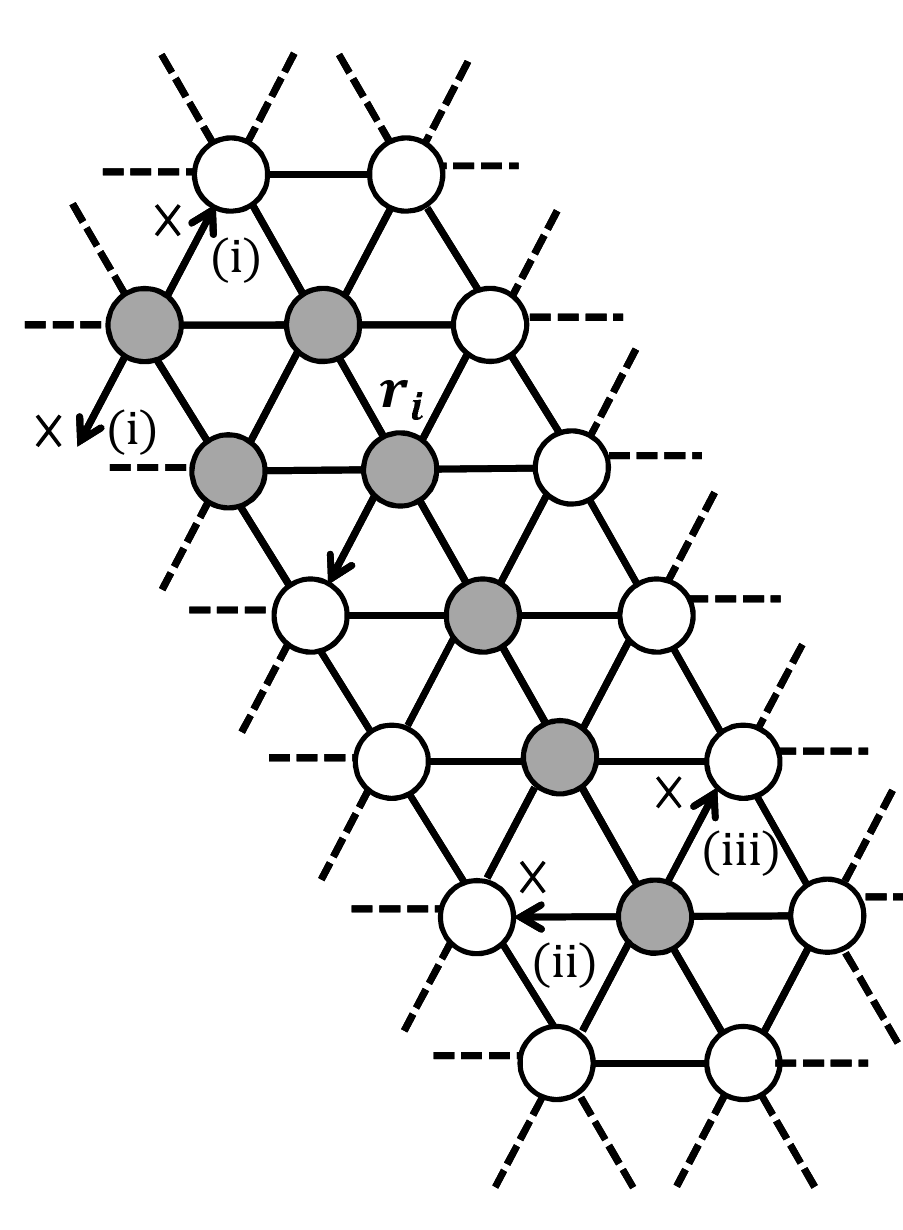}
			\caption{A configuration such that only robot $r_i$ with three adjacent robot nodes SE, W, and NW
				can leave the current node
				((i): by Fig.\,\ref{fig:case5-2counter3},
				(ii): by Lemma \ref{lem:noNWtoW},
				(iii): by Lemma \ref{lem:noNWtoNE}).}
			\label{fig:case5-2Mid2}
		\end{center}
	\end{minipage}
	\begin{minipage}{0.01\hsize}
		\hspace{2mm}
	\end{minipage}
	\begin{minipage}{0.48\hsize}
		\begin{center}
			\includegraphics[scale=0.50]{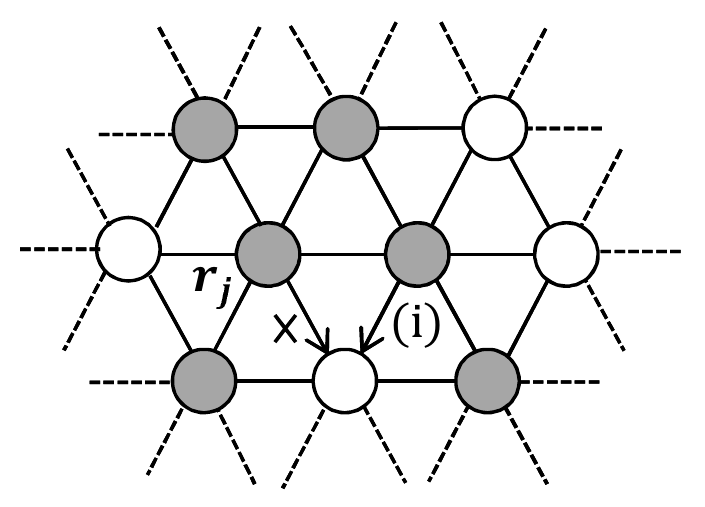}
			\caption{An example showing that a robot $r_j$ with four adjacent robot nodes E, SW, NW, and NE 
				cannot move to SE
				((i): by Fig.\,\ref{fig:case5-2Mid2}).}
			\label{fig:case5-2counter4}
		\end{center}
	\end{minipage}
\end{figure} 

\begin{figure}[t!]
	\begin{minipage}{0.48\hsize}
		\begin{center}
			\includegraphics[scale=0.55]{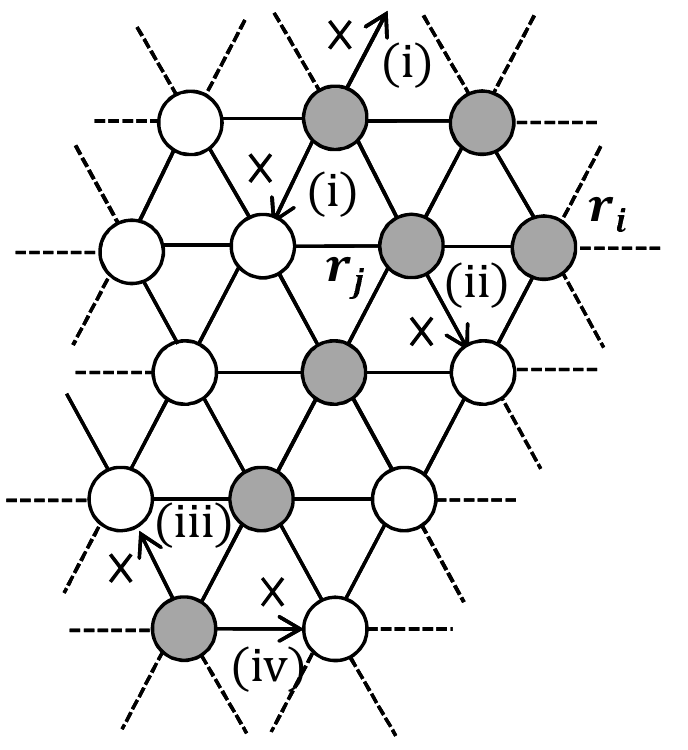}
			\caption{A configuration such that only robot $r_i$ with two adjacent robot nodes  W and NW
				can leave the current node
				((i): by Fig.\,\ref{fig:case5-2counter3},
				(ii): by Fig.\,\ref{fig:case5-2counter4},				
				(iii): by Fig.\,\ref{fig:assumptionCounter} (a)
				(iv): by Lemma \ref{lem:noNEtoE}).}
			\label{fig:case5-2Mid3}
		\end{center}
	\end{minipage}
	\begin{minipage}{0.01\hsize}
		\hspace{2mm}
	\end{minipage}
	\begin{minipage}{0.48\hsize}
		\begin{center}
			\includegraphics[scale=0.55]{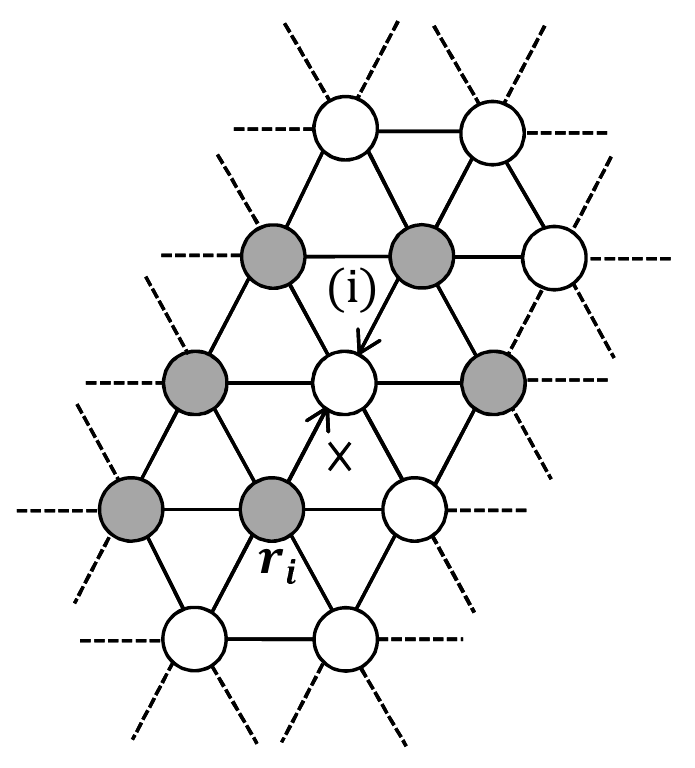}
			\caption{An example showing that a robot $r_i$ with two adjacent robot nodes W and  NW 
				cannot move to NE
				((i): assumption of Case 5-2).}
			\label{fig:case5-2counter5}
		\end{center}
	\end{minipage}
\end{figure} 

\if()
\begin{figure}[t!]
	\begin{minipage}{0.48\hsize}
		\begin{center}
			\includegraphics[scale=0.525]{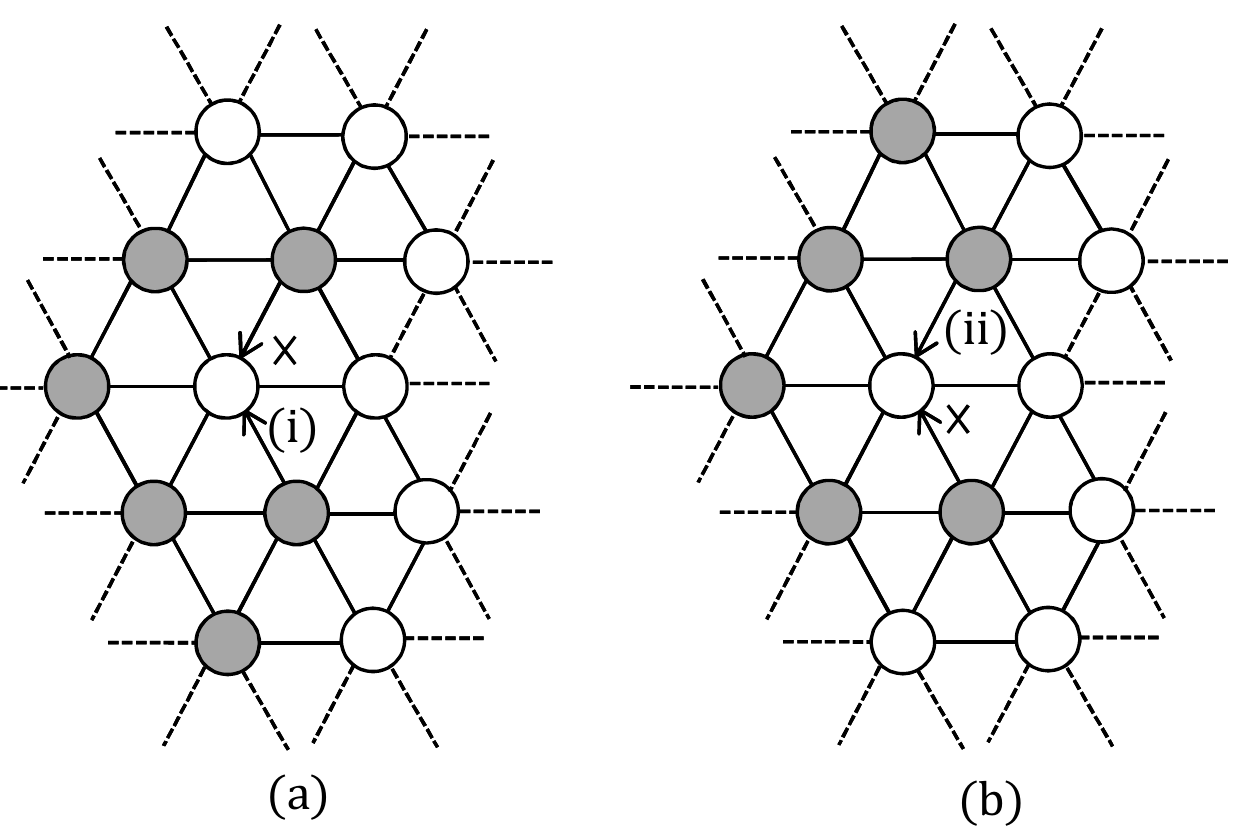}
			\caption{Examples showing that a robot with one adjacent robot node W
				cannot leave the current node
				((i): by Fig.\,\ref{fig:case5-2counter2},
				(ii): by Fig.\,\ref{fig:case5-2counter5}).
			}
			\label{fig:case5-2counter6}
		\end{center}
	\end{minipage}
	\begin{minipage}{0.01\hsize}
		\hspace{2mm}
	\end{minipage}
	\begin{minipage}{0.48\hsize}
		\begin{center}
			\includegraphics[scale=0.525]{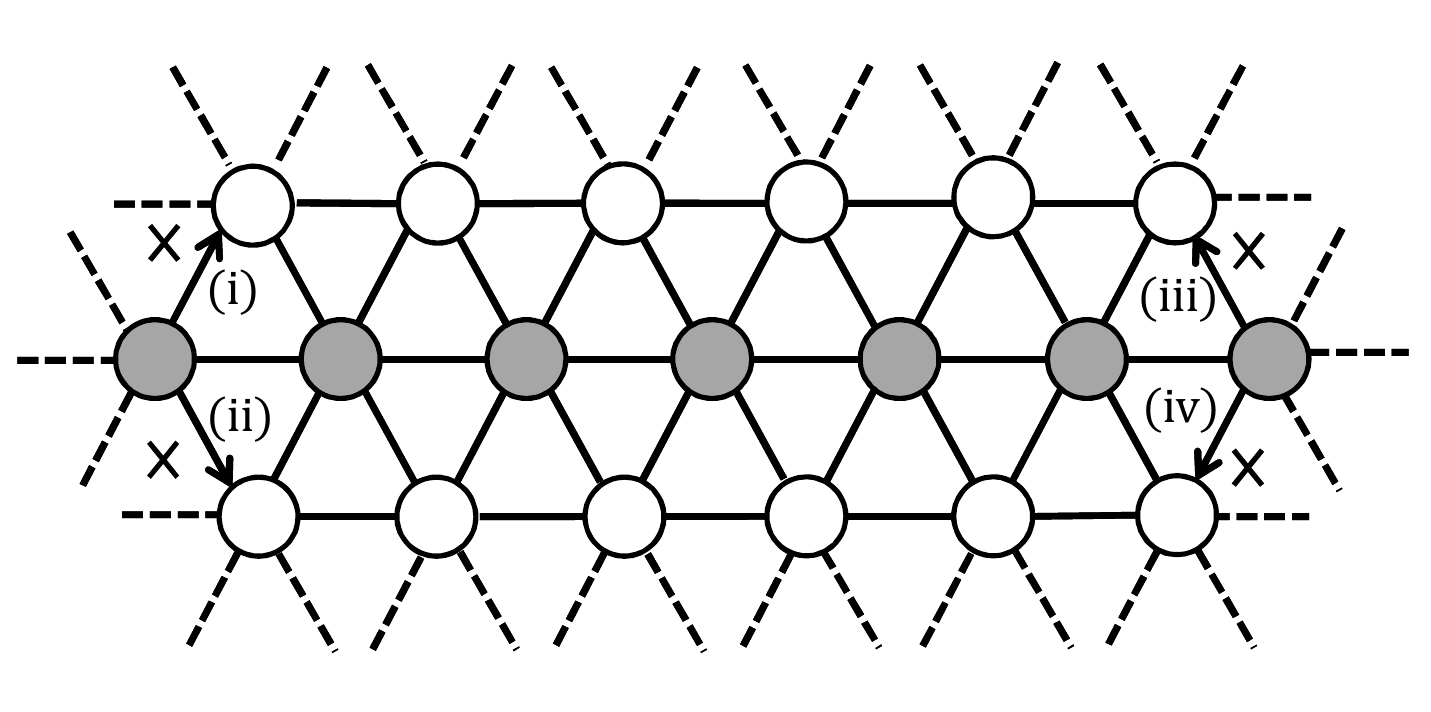}
			\caption{An unsolvable configuration when 
				a robot with two adjacent robot nodes W and SE moves to SW
				((i): by Fig.\,\ref{fig:case5-2counter} (a),
				(ii): by Fig.\,\ref{fig:case5-2counter} (b),
				(iii): by Fig.\,\ref{fig:case5-2counter6} (b),
				(iv): by Fig.\,\ref{fig:case5-2counter6} (a)).}
			\label{fig:case5-2final}
		\end{center}
	\end{minipage}
\end{figure} 
\fi

\begin{figure}[t!]
	\begin{center}
		\includegraphics[scale=0.525]{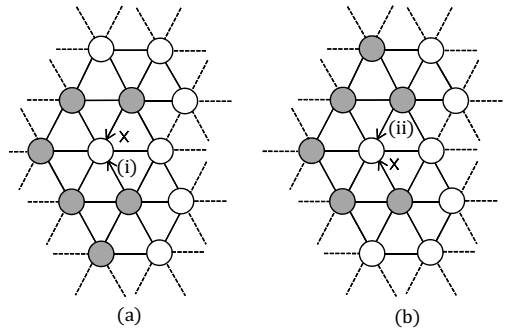}
		\caption{Examples showing that a robot with one adjacent robot node W
			cannot leave the current node
			((i): by Fig.\,\ref{fig:case5-2counter2},
			(ii): by Fig.\,\ref{fig:case5-2counter5}).}
		\label{fig:case5-2counter6}
	\end{center}
\end{figure}

\begin{figure}[t!]
	\begin{center}
		\includegraphics[scale=0.475]{eps/case5-2final}
		\caption{An unsolvable configuration when 
			a robot with two adjacent robot nodes W and SE moves to SW
			((i): by Fig.\,\ref{fig:case5-2counter} (a),
			(ii): by Fig.\,\ref{fig:case5-2counter} (b),
			(iii): by Fig.\,\ref{fig:case5-2counter6} (b),
			(iv): by Fig.\,\ref{fig:case5-2counter6} (a)).}
		\label{fig:case5-2final}
	\end{center}
\end{figure}

\begin{lemma}
	\label{lem:noSWtoW}
	A robot with one adjacent robot node \textit{SW} cannot move to node W. 
\end{lemma}


Thus, by Proposition \ref{pro:SEtoSW}-(a), Lemmas \ref{lem:intermediate}, \ref{lem:noSWtoSE}, \ref{lem:noNEtoE}, and \ref{lem:noNWtoNE},
robots cannot achieve gathering from a configuration of Fig.\,\ref{fig:slash} (b)
and we have the theorem. 
\end{proof}

\if()

Similarly to the proofs of Lemmas \ref{lem:noNWtoW} and \ref{lem:noSWtoSE},
in the remaining cases we can have the following lemmas 
by showing several prohibited robot behaviors and 
a configuration from which robots cannot achieve gathering 
(the detailed proofs are given in \cite{aiXive}).

\begin{lemma}
	\label{lem:noNEtoE}
		A robot with one adjacent robot node \textit{NE} cannot move to node E. 
	\end{lemma}

\begin{lemma}
	\label{lem:noNWtoNE}
	A robot with one adjacent robot node \textit{NW} cannot move to node \textit{NE}. 
\end{lemma}

\begin{lemma}
	\label{lem:noSWtoW}
	A robot with one adjacent robot node \textit{SW} cannot move to node W. 
\end{lemma}

Thus, by Proposition \ref{pro:SEtoSW}-(a) and 
Lemmas \ref{lem:intermediate}, \ref{lem:noSWtoSE}, \ref{lem:noNEtoE}, 
and \ref{lem:noSWtoW}, robots cannot achieve gathering from the configuration of
Fig.\,\ref{fig:slash} (b).
This contradicts the assumption that 
there exists a collision-free algorithm $\cal{A}$ to solve the gathering problem.
Therefore, the theorem follows. 

\fi

\if()
\par

\textit{Remark.}
We showed impossibility of gathering for seven mobile robots 
which requires that robots form 
a hexagon with radius 1.
This result also holds for gathering of the specific number of  robots such that 
they form a hexagon with radius of a positive integer after gathering.
\fi

\section{Robots with visibility range 2}
\label{sec:algo}
In this section, for robots with visibility range 2, 
we propose a collision-free algorithm to solve the gathering problem 
from any connected initial configuration.

\subsection{Proposed algorithm}
The basic idea is that each robot firstly determines the \textit{base node} that is 
the rightmost robot node within its visibility range and then
it moves toward the base node to achieve gathering.
First, we explain how to determine the base (or rightmost) robot.
For explanation, in the following we assume that each robot $r_i$
recognizes that it is located at an origin and it assigns labels to each node within
its visibility range like Fig.\,\ref{fig:label}.
In the figure, the first (resp., second) element of each label is called 
the \textit{$x$-element} (resp., \textit{$y$-element})\footnote{
	Labels are assigned for explanation and they are a little different from the coordinate system.
	For example, the difference between labels (0,0) and (2,0) is 2 but 
	the distance between node (0,0) and node (2,0) is 1.}.
Then, $r_i$ determines the robot node with the largest $x$-element as the base node 
(possibly the robot node where $r_i$ itself stays).
If several robot nodes have the largest $x$-element, 
$r_i$ does not determine the base node at that time and 
waits at the current node until the configuration changes.
As exceptions, if node (4,0) is an empty node and nodes (3,1) and (3,-1) are robot nodes, 
$r_i$ determines node (4,0) as the base node to avoid the configuration such that 
no robot determines a base node and each robot waits at the current node.
In addition, if robot nodes (1,1) and (1,-1) have the largest $x$-element 
among all the labels of robot nodes within $r_i$'s visibility range, and 
$r_i$ moves to node (2,0) so that it becomes a base.
Examples are given in Fig.\,\ref{fig:base}.

\if()
\begin{figure}[t!]
	\begin{center}
	\includegraphics[scale=0.7]{{eps/avoidCollision}}
	\caption{An example to avoid a collision using ordinal numbers.}
	\label{fig:avoidCollision}
	\end{center}
\end{figure}
\fi

\begin{figure}[t!]
	\centering
	\includegraphics[scale=0.475]{{eps/label}}
	\caption{Assignment of labels.}
	\label{fig:label}
\end{figure}

\begin{figure}[t!]
	\centering
	\includegraphics[scale=0.7]{{eps/base}}
	\caption{Examples of how to determine  the base nodes
		((a): node $v_b$ is the base node, 
		(b): $r_i$ does not determine the base node,
		(c): $r_i$ determines $v_b$ as the base node and moves there).}
	\label{fig:base}
\end{figure}

\begin{figure}[t!]
	\centering
	\includegraphics[scale=0.725]{{eps/movingRule}}
	\caption{Movement rules 
		((a): candidate nodes to visit, (b): ordinal numbers).}
	\label{fig:rule}
\end{figure}

Next, we explain how to achieve gathering based on the base node.
Robots consider the base node as the rightmost node of a gathering-achieved configuration and 
they basically move east 
on a triangular grid 
with avoiding a collision and an unconnected configuration.
Concretely, if the label of the base node from robot $r_i$ is
(2,-2), (3,-1), (4,0), (3,1), or (2,2), 
it moves to one of adjacent nodes as indicated in Fig.\,\ref{fig:rule} (a) using 
ordinal numbers in Fig.\,\ref{fig:rule} (b).
That is, among the candidate nodes that $r_i$ may visit in the next cycle,
$r_i$ moves to the empty adjacent node with the smallest ordinal number.
If several robots try to move to the same node $v_j$, 
the robot staying at the node with the largest ordinal number moves to $v_j$. 
If all the candidate nodes are robot nodes, $r_i$ stays at the current node.
For example, in Fig.\,\ref{fig:avoidCollision},
robots $r_i$ and $r_j$ consider the common node $v_b$ as the base node,
$r_i$ (resp., $r_j$) has two candidate nodes $v_n$ and $v_\ell$ 
(resp., $v_m$ and $v_\ell$) to visit,
$v_n$ (resp., $v_m$) is a robot node, and hence it tries to visit $v_\ell$ (resp., $v_\ell$).
In this case, since the ordinal number 4 of the node where  $r_i$ stays is larger than 
the ordinal number 3 of the node where  $r_j$ stays, 
$r_i$ moves to $v_\ell$ and $r_j$ stays at the current node.
If two robots consider the common node as the base node like the above  example,
they can share the common ordinal numbers and can avoid a collision or an unconnected configuration.
However, it is possible that 
some two robots consider different robot nodes as their base nodes due to 
their  limited visibility range, which may cause a collision or an unconnected configuration.
For example, in Fig.\,\ref{fig:avoidCollision2},
robot $r_i$ considers $v'_b$ as the base node but $r_j$ considers $v_b$ as the base node, and 
they try to move to the same node $v_\ell$ according to the movement rule.
In this case, the robot with the smaller $x$-element of the node label
moves to the node and the other robot stays at the current node.
Hence, in Fig.\,\ref{fig:avoidCollision2},
$r_i$ moves to $v_\ell$ and $r_j$ stays at the current node.
Moreover, only with the movement rule in Fig.\,\ref{fig:rule},
no robot leaves the current node in the configuration in Fig.\,\ref{fig:avoidStandstill}.
In this case, as a special behavior, 
if the label of the base node from robot $r_i$ is (3,1),
nodes (1,1), (2,0), and (1,-1) are robot nodes, and node (-1,1) is an empty node,
$r_i$ moves to the northwest adjacent node (-1,1) so that 
robot $r_j$ staying at $r_i$'s southeast adjacent node (1,-1) can move to the node 
where $r_i$ is currently staying.
When robots reach a configuration such that no robot leaves the current node,
the configuration is one solution of the gathering problem.	

\begin{figure}[t!]
	\begin{minipage}{0.48\hsize}
		\centering
		\includegraphics[scale=0.7]{{eps/avoidCollision}}
		\caption{An example to avoid a collision using ordinal numbers.}
		\label{fig:avoidCollision}
	\end{minipage}
	\begin{minipage}{0.01\hsize}
		\hspace{2mm}
	\end{minipage}
	\begin{minipage}{0.48\hsize}
		\centering
		\includegraphics[scale=0.7]{{eps/avoidCollision2}}
		\caption{An example to avoid a collision using $x$-elements.} 
		\label{fig:avoidCollision2}
	\end{minipage}
\end{figure}

\begin{figure*}[t!]
	\begin{minipage}{0.3\hsize}
		\centering
		\includegraphics[scale=0.69]{{eps/avoidStandstill}}
		\caption{An example to avoid a standstill.}
		\label{fig:avoidStandstill}
	\end{minipage}
	\begin{minipage}{0.01\hsize}
		\hspace{2mm}
	\end{minipage}
	\begin{minipage}{0.468\hsize}
		\centering
		\includegraphics[scale=0.68]{{eps/algoEX}}
		\caption{An example of the algorithm execution.}
		\label{fig:algoEX}
	\end{minipage}
\end{figure*}

\if()
\begin{figure*}[t!]
		\begin{center}
	\includegraphics[scale=0.7]{{eps/avoidCollision2}}
	\caption{An example to avoid a collision using $x$-elements of node labels.}
	\label{fig:avoidCollision2}
	\end{center}
\end{figure*}
\fi

\if()
\begin{figure*}[t!]
	\begin{center}
		\includegraphics[scale=0.69]{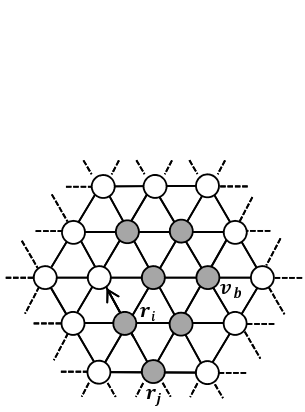}
		\caption{An example to avoid a standstill.}
		\label{fig:avoidStandstill}
	\end{center}
\end{figure*}
\begin{figure*}[t!]
	\begin{center}
		\includegraphics[scale=0.68]{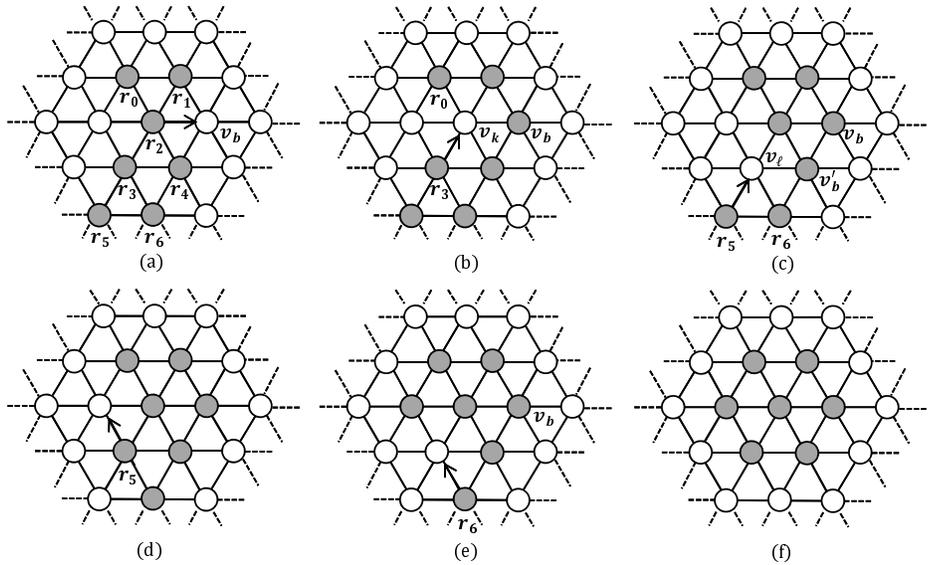}
		\caption{An example of the algorithm execution.}
		\label{fig:algoEX}
	\end{center}
\end{figure*}
\fi

An example of the algorithm execution is given in Fig.\,\ref{fig:algoEX}.
From (a) to (b), since $r_2$'s northeast and southeast adjacent nodes are robot nodes and 
no other robot node has larger $x$-element than that of the node where $r_2$ stays,
it moves to east adjacent node $v_b$.
From (b) to (c), robots $r_0$ and $r_3$ consider $v_b$ as the common base node and 
they try to move to the empty node $v_k$ with the smallest ordinal number
among candidate empty nodes.
In this case, since
the ordinal number of the node where  $r_3$ stays is larger than 
that of the node where   $r_0$ stays,
$r_3$ moves to $v_k$ and $r_0$ stays at the current node.
From (c) to (d), $r_5$ (resp., $r_6$) considers $v'_b$  (resp., $v_b$) as the base node and 
they try to move to node $v_\ell$.
In this case, since the $x$-element of the node that where  $r_5$  stays is smaller than 
that of the node where  $r_6$ stays,
$r_5$ moves to $v_\ell$ and $r_6$ stays at the current node.
From (d) to (e), as a special behavior, 
robot $r_5$ moves to the northwest adjacent robot node so that 
$r_6$ can move to the node where  $r_5$ is currently staying.
From (e) to (f), robot 	$r_6$ considers $v_b$ as the base node 
and it moves to the northeast adjacent node.
Then, robots achieve gathering.

The pseudocode of the proposed algorithm is described in Algorithm \ref{algo}.
In the following, we explain several robot behaviors 
that avoid a collision or an unconnected configuration.
The behavior of robot $r_i$ for the case that, 
the label of the base node is (2,0) but the node is an empty node, 
is described in lines 1 -- 3.
In this case, $r_i$ tries to move to node (2,0).
However, if 
$r_i$'s west adjacent node (-2,0) is a robot node and $r_i$ moves to the base node (2,0), 
the configuration may become unconnected (Fig.\,\ref{fig:algo1} (a)).
Hence, in this case $r_i$ moves to  node (2,0) when
$r_i$'s northwest or southwest adjacent node is also a robot node 
(Fig.\,\ref{fig:algo1} (b)).

The  behavior of robot $r_i$ for the case that 
the label of the base node is (4,0) is described in lines 5 --	 9.
In this case, if node (2,0) is an empty node, 
$r_i$ tries to move to the node.
However, if 
$r_i$'s southwest adjacent node (-1,-1) is a robot node and $r_i$ moves to node (2,0), 
the configuration may become unconnected (Fig.\,\ref{fig:algo2} (a)).
Hence, in this case $r_i$ moves to node (2,0) when 
its southeast adjacent node (1,-1) is also a robot node (Fig.\,\ref{fig:algo2} (b)).

\begin{figure}[t!]
	\centering
	\includegraphics[scale=0.68]{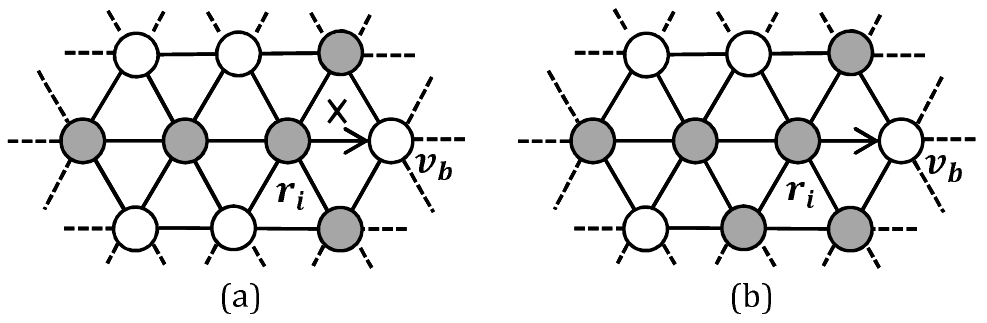}
	\caption{Behavior of robot $r_i$ when 
		the label of the base node $v_b$ is (2,0) but the node is an empty node.
	}
	\label{fig:algo1}
\end{figure}

\begin{figure}[t!]
	\centering
	\includegraphics[scale=0.68]{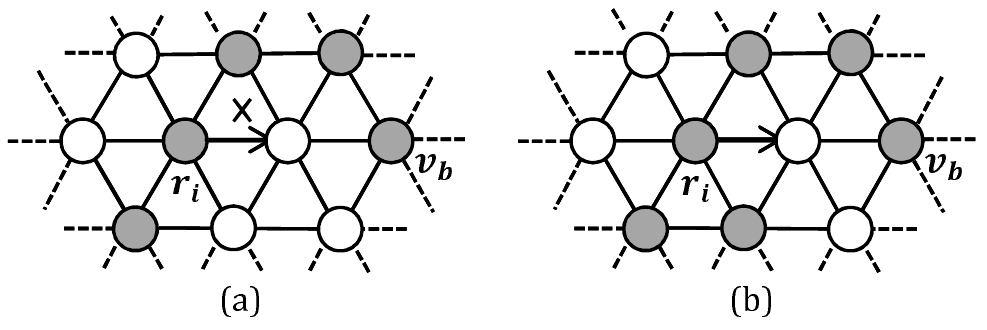}
	\caption{Behavior of robot $r_i$ when 
		the label of the base node $v_b$ is (4,0).		}
	\label{fig:algo2}
\end{figure}

\begin{figure}[t!]
	\centering
	\includegraphics[scale=0.9]{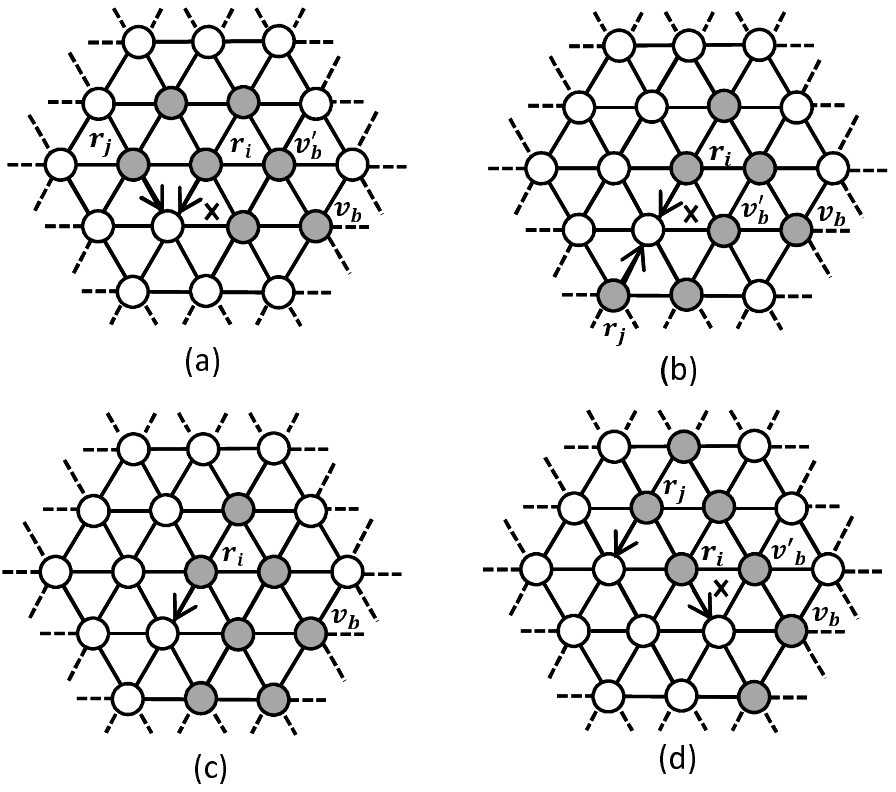}
	\caption{Behavior of robot $r_i$ when 
		the label of the base node $v_b$ is (3,-1)
		($v'_b$: the base node for robot $r_j$).}
	\label{fig:algo3}
\end{figure}

\begin{figure*}[t!]
	\centering
	\includegraphics[scale=0.75]{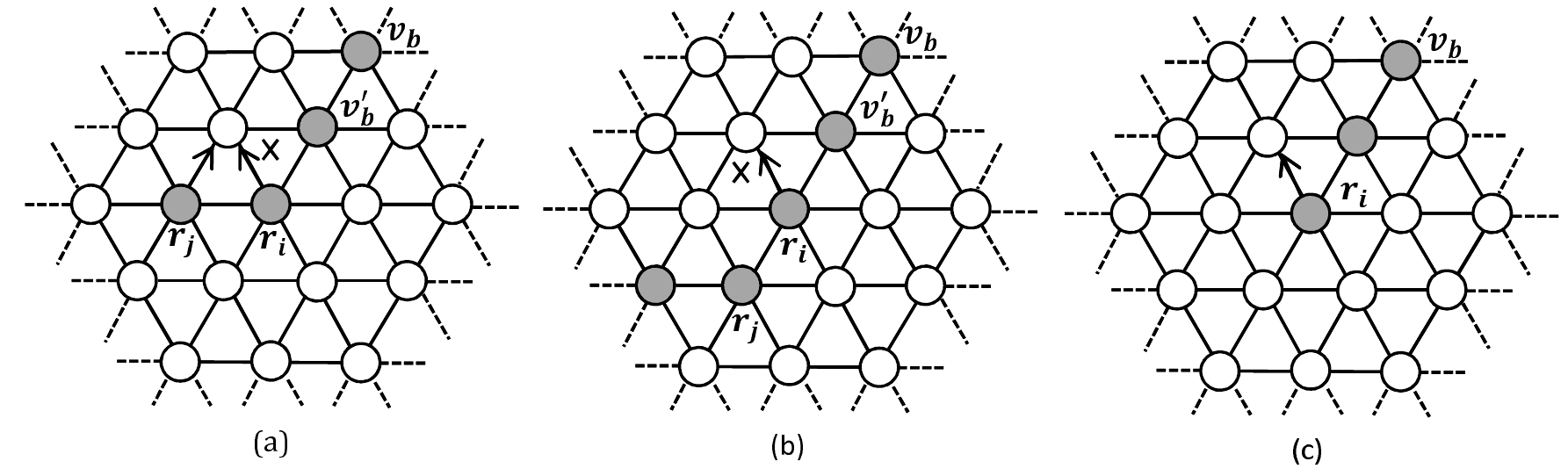}
	\caption{Behavior of robot $r_i$ when 
		the label of the base node $v_b$ is (2,2)
		($v'_b$: the base node for robot $r_j$).}
	\label{fig:algo4}
\end{figure*} 

The behavior of robot $r_i$ for the case that 
the label of the base node is (3,-1) is described in lines 11 -- 15.
In this case, if  nodes (2,0), (1,-1) and (1,1) are robot nodes and 
node (-1,-1) is an empty node, 
$r_i$ tries to move to its southwest adjacent node (-1,-1) so that 
the robot staying at node (1,1) could move to node (0,0) where 
$r_i$ is currently staying.
However, due to the limited visibility range, it is possible that 
$r_i$ and some robot $r_j$ consider different nodes as base nodes,
$r_j$ staying at node (-2,0) or (-2,-2) tries to move to node (-1,-1), and a collision occurs (Fig.\,\ref{fig:algo3} (a), (b)).
Hence, in this case 
$r_i$ moves to node (-1,-1) when 
nodes (-2,0), (-2,-2), and (-1,1) are empty nodes (Fig.\,\ref{fig:algo3} (c)).
In addition, if node (1,-1) is an empty node, 
$r_i$ tries to move to the node.
Then, it is possible that 
$r_i$ and some robot $r_j$ consider different nodes are base nodes,
$r_j$ staying at node (-1,1) tries to move to node (-2,0), and the configuration become unconnected
(Fig.\,\ref{fig:algo3} (d)).
Hence, in this case 
$r_i$ moves to node (1,-1) when 
node (0,-2) is an empty node.

\begin{algorithm*}[t!]                     
	\caption{Proposed algorithm}         
	\label{algo}   
	\begin{algorithmic}[1] 
		\small
		\IF{(node (2,0) is an empty node) $\land$ (nodes (1,1) and (1,-1) are robot nodes)
			$\land$ (the other robot nodes have $x$-elements of the labels at most 0)}
		\STATE /*The base node is (2,0) but it is am empty node*/
		\STATE \textbf{if} (node (-2,0) is an empty node) $\lor$ 
		((node (-2,0) is a robot node) $\land$ (node (-1,1) or (-1,-1) is a robot node))
		\textbf{then} move to the east adjacent node (2,0)
		\STATE
		
		\ELSIF{(node label of the base node is (4,0)) $\lor$ 
			((node (4,0) is an empty node) $\land$ 
			(nodes (3,1) and (3,-1) are robot nodes))}
		\STATE /*The base node is (4,0)*/				
		\STATE \textbf{if} (node (2,0) is an empty node) $\land$
		((nodes (-1,1), (-2,0), and (-1,-1) are empty nodes) $\lor$
		(node (1,-1) is a robot node and nodes (-2,0) and (-1,1) are empty nodes) $\lor$
		(node (1,1) is a robot node and nodes (-2,0) and (-1,-1) are empty nodes) $\lor$
		(nodes (1,-1), (-1,-1), and (-2,0) are robot nodes and 
		node (-1,1) is an empty node) $\lor$ 
		(nodes (-2,0), (-1,1) and (1,1) are robot nodes and node (-1,-1) is an empty node))
		\textbf{then} move to the east adjacent node (2,0)
		\STATE \textbf{else if} (node (2,0) is a robot node) $\land$
		(node (1,1) is an empty node) $\land$ 
		(nodes (-2,0) and (-1,1) are empty nodes) $\land$
		((nodes (-1,-1) and (2,2) are empty nodes) $\lor$
		(nodes (2,2), (3,1), (3,-1), and (-2,-2) are robot nodes))
		\textbf{then}  move to the northeast robot node (1,1)
		\STATE \textbf{else if} (nodes (2,0) and (1,1) are robot nodes) $\land$
		(nodes (1,-1) is an empty node) $\land$
		(nodes (-1,-1) (-2,0), (-1,1), and (2,-2) are empty nodes) $\land$
		((node (1,1) is a robot node) $\lor$ (node (2,2) is a robot node))
		\textbf{then} move to the southeast adjacent node (1,-1)
		
		\STATE
		\ELSIF{node label of the base node is (3,-1)}
		\STATE /*The base node is (3,-1)*/	
		\STATE \textbf{if} (node (1,-1) is an empty node) $\land$ 
		(nodes (-1,-1) and (0,-2) are empty nodes) $\land$ 
		((nodes (-2,0) and (-1,1) are empty nodes) $\lor$ 
		(nodes (-1,1) and (1,1) are robot nodes and node (0,2) is an empty node))
		\textbf{then} move to the southeast adjacent node (1,-1)
		\STATE \textbf{else if} (node (1,-1) is a robot node) $\land$
		(node (2,0) is an empty node) $\land$ 
		(node (-1,1) is an empty node) $\land$ 
		((node (-2,0) is an empty node) $\lor$ 
		(nodes (-2,0) and (-1,-1) are robot nodes))
		\textbf{then} move to the east adjacent node (2,0)
		\STATE \textbf{else if} (nodes (1,-1) and (2,0) are robot nodes) $\land$
		(node (1,1) is a robot node) $\land$
		(node (-1,-1) is an empty node) $\land$
		(nodes (-2,0) and (-2,-2) are  empty node) 
		\textbf{then} move to the southwest node (-1,-1)
		
		\STATE
		\ELSIF{node label of the base node is (2,-2)}
		\STATE /*The base node is (2,-2)*/	
		\STATE \textbf{if} (node (-1,-1) is an empty node) $\land$
		(nodes (-2,0), (-3,-1), and (-1,1) are empty nodes)
		\textbf{then} move to the southwest adjacent node (-1,-1)
		
		\STATE
		\ELSIF{node label of the base node is (3,1)}
		\STATE /*The base node is (3,1)*/			
		\STATE \textbf{if} (node (1,1) is an empty node) $\land$
		((nodes (-1,1), (-2,0), (-1,-1) are empty nodes) $\lor$
		(nodes (1,-1) and (-1,-1) are robot nodes and 
		nodes (0,-2) and (-1,1) are empty node))
		\textbf{then} move to the northeast adjacent node (1,1)
		\STATE \textbf{else if} (node (1,1) is a robot node) $\land$
		(node (2,0) is an empty node) $\land$
		((nodes (-2,0) and (-1,-1) are empty nodes) $\lor$
		(node (-1,-1) is an empty node and nodes (-2,0) and (-1,1) are robot nodes))
		\textbf{then} move to the east adjacent node (2,0)
		\STATE \textbf{else if} (nodes (1,1) and (2,0) are robot nodes) $\land$
		(node (1,-1) is a robot node) $\land$ (node (1,-1) is an empty node)
		$\land$ (nodes (-2,0), and (-2,2) are empty nodes)
		\textbf{then} move to the northwest adjacent node (-1,1)
		
		\STATE
		\ELSIF{node label of the base node is (2,2)}
		\STATE /*The base node is (2,2)*/			
		\STATE \textbf{if} (node (-1,1) is an empty node) $\land$ 
		(nodes (-3,1), (-2,0), and (-1,-1) are empty nodes) 
		\textbf{then} move to its northwest adjacent node (-1,1)  	
		\if()
		\STATE \textbf{else if} (nodes (2,0) is a robot node) $\land$
		(node (1,-1) is an empty node) $\land$ 
		(nodes (2,-2), (-1,-1), (-2,0) are empty nodes) $\land$
		((node (1,1) is a robot node) $\lor$
		(node (2,2) is a robot node))
		\textbf{then} move to the southeast adjacent node (1,-1)	
		\fi
		
		\STATE
		\ELSIF{(node label of the base node is (0,0) or (2,0) or (1,-1) or (1,1)) $\lor$
			(there is no base node)}
		\STATE /*Robot $r_i$ is close to the base node 
		and it does not need to leave the current node*/					
		\STATE stay at the current node 
		\ENDIF

	\end{algorithmic}
\end{algorithm*}

The behavior of robot $r_i$ for the case that 
the label of the base node is (2,2) is described in lines 27 -- 29.
In this case, if node (1,1) is a robot node and 
node (-1,1) is an empty node, 
it tries to move to node (-1,1).
However, due to the limited visibility range, it is possible that 
$r_i$ and some robot $r_j$ consider different nodes are base nodes,
$r_j$ staying at node (-2,0) tries to move to node (-1,1), and  a collision occurs (Fig.\,\ref{fig:algo4} (a)),
or $r_j$ staying at node (-1,-1) does not leave the current node 
and the configuration becomes unconnected (Fig.\,\ref{fig:algo4} (b)).
Hence, in this case 
$r_i$ moves to node (-1,-1) when
nodes (-2,0) and (-1,-1) are empty nodes (Fig.\,\ref{fig:algo4} (c)).

Although there still exist several robot behaviors that 
avoid a collision or an unconnected configuration, we omit the detail.

\subsection{Correctness}

The correctness of the proposed algorithm has been evaluated by computer simulations. 
By the simulations, we confirmed that robots which execute the proposed algorithm can achieve gathering from 
all possible connected initial configurations (3652 patterns in total) in the fully synchronous (FSYNC) model. 
Thus, we have the following theorem.

\begin{theorem}
For robots with visibility range 2,
the proposed algorithm solves the gathering problem from any connected initial configuration in the FSYNC model.
\end{theorem}

\section{Conclusion}\label{conclusion}
In this paper, we considered the gathering problem of seven autonomous mobile robots 
on triangular grid graphs. 
First, for robots with visibility range 1,
we showed that no collision-free algorithm exists for the gathering problem.
Next, for robots with visibility range 2,
we proposed a collision-free algorithm to solve the problem
from any connected initial configuration.
This algorithm is optimal in terms of visibility range.

There are four possible future works as follows. 
First, we will complete a theoretical proof of correctness for the proposed algorithm in Section \ref{sec:algo}.
\if()
Second, we will consider robots that agree only on chirality
(i.e., do not agree on the direction and orientation of the $x$-axis).
In this case, we conjecture that
robots cannot solve the problem even with visibility range 2 and 
they need more visibility range.
\fi
Second, we will consider the relaxed version of connected initial configuration such that 
the visibility relationship among robots constitutes one connected graph.
Third, we will consider gathering for different number of robots.
Lastly, we consider other problems such as the pattern formation problem 
for autonomous mobile robots on triangular grids.

\if()
Another future work is to solve the problem 
for robots with visibility range 2 in the semi-synchronous (SSYNC) model in which,
in each cycle, a subset of the robots are activated 
and they execute the cycle synchronously.
We conjecture that
the proposed algorithm in this paper can solve the problem also in the SSYNC model,
and we will complete the proof of correctness.
\fi

\section*{Acknowledgement}
This work was partially supported by JSPS KAKENHI Grant Number 18K18029, 18K18031, 19K11823, 20H04140, and 20KK0232;
the Hibi Science Foundation; and Foundation of Public Interest of Tatematsu.

\bibliographystyle{unsrt}
\bibliography{ref}

\end{document}